\documentclass{new_tlp}
\usepackage{amsmath}
\usepackage{graphicx}
\usepackage{multirow}
\usepackage{todonotes}

\usepackage{amssymb}
\usepackage{xspace}
\usepackage{url}
\usepackage{booktabs}
\usepackage{enumitem}
\usepackage{hhline}
\usepackage{multirow}
\usepackage{caption}
\usepackage{subcaption}
\usepackage[ruled,vlined,linesnumbered]{algorithm2e}
\usepackage{refcount}

\newcommand{\system}{\textsc{qasp}\xspace}
\newcommand{\pyqasp}{\textsc{pyqasp}\xspace}
\newcommand{\stunst}{\textsc{st-unst}\xspace}
\newcommand{\qasp}{\system}

\newcommand{\ASPQ}{ASP(Q)\xspace}
\newcommand{\fix}{\mathit{fix}}

\newcommand{\Int}{\mathit{Int}}
\newcommand{\CH}{\mathit{CH}}
\newcommand{\CNF}{\mathit{CNF}}
\newcommand{\as}{\mathit{AS}}
\newcommand{\nop}[1]{}
\newcommand{\myparagraph}[1]{\smallskip\noindent\textit{#1}\xspace}

\newcommand{\qaspv}[2]{$\textit{\system}^{\mathit{#1}}_{\mathit{#2}}$\xspace}
\newcommand{\pyqaspv}[2]{$\textit{\pyqasp}^{\mathit{#1}}_{\mathit{#2}}$\xspace}

\newtheorem{thm}{Theorem}
\newtheorem{definition}{Definition}
\newtheorem{prop}{Proposition}
\newtheorem{example}{Example}[section]

  \title[An efficient solver for \ASPQ]
        {An efficient solver for \ASPQ%
\thanks{This work was partially supported by MUR under PRIN project PINPOINT Prot. 2020FNEB27, CUP H23C22000280006, and PNRR project PE0000013-FAIR, Spoke 9 - Green-aware AI – WP9.1.
}}
    \author[W. Faber, G. Mazzotta, F. Ricca]
         {WOLFGANG FABER$^1$ \and GIUSEPPE MAZZOTTA$^2$ \and FRANCESCO RICCA$^2$\\
         $^1$Alpen-Adria Universität Klagenfurt, Austria\\
         $^2$University of Calabria, Rende, Italy\\
         }
         
\jdate{}
\pubyear{}
\pagerange{}
\doi{}

\begin{document}

\label{firstpage}

\maketitle

  \begin{abstract}
    Answer Set Programming with Quantifiers ASP(Q) extends Answer Set Programming (ASP) to allow for declarative and modular modeling of problems from the entire polynomial hierarchy. The first implementation of ASP(Q), called qasp, was based on a translation to Quantified Boolean Formulae (QBF) with the aim of exploiting the well-developed and mature QBF-solving technology. However, the implementation of the QBF encoding employed in qasp is very general and might produce formulas that are hard to evaluate for existing QBF solvers because of the large number of symbols and sub-clauses. In this paper, we present a new implementation that builds on the ideas of qasp and features both a more efficient encoding procedure and new optimized encodings of ASP(Q) programs in QBF. The new encodings produce smaller formulas (in terms of the number of quantifiers, variables, and clauses) and result in a more efficient evaluation process. An algorithm selection strategy automatically combines several QBF-solving back-ends to further increase performance. An experimental analysis, conducted on known benchmarks, shows that the new system outperforms qasp.
  \end{abstract}

  \begin{keywords}
    ASP with Quantifiers, Quantified Boolean Formulas, Well-founded semantics
  \end{keywords}

\section{Introduction}
Answer Set Programming (ASP)~\cite{DBLP:journals/cacm/BrewkaET11,DBLP:journals/ngc/GelfondL91} is a popular logic programming paradigm based on the stable models semantics, offering the capabilities to $(i)$ modeling search and optimization problems in a declarative (and often compact) way
and $(ii)$ solving them using efficient systems~\cite{DBLP:conf/ijcai/GebserLMPRS18} that can handle real-world problems~\cite{DBLP:journals/aim/ErdemGL16,DBLP:journals/jair/GebserMR17}. 

Despite being very effective in modeling and solving problems in NP~\cite{DBLP:journals/jair/GebserMR17}, the first level of the Polynomial Hierarchy (PH), ASP is less practical when one has to approach problems beyond NP.
The existing programming techniques, such as \emph{saturation}~\cite{DBLP:journals/amai/EiterG95,DBLP:journals/csur/DantsinEGV01}, which allow for encoding with ASP problems that belong to the second level of the PH, are not very intuitive. Moreover,
the expressive power of ASP does not span the entire PH.

Recently, these shortcomings of ASP have been overcome by the introduction of language extensions that expand the expressivity of ASP~\cite{DBLP:journals/tplp/BogaertsJT16,DBLP:journals/tplp/FandinnoLRSS21,DBLP:journals/tplp/AmendolaRT19}.
Among these, Answer Set Programming with Quantifiers \ASPQ extends ASP, allowing for declarative and modular modeling of problems of the entire PH~\cite{DBLP:journals/tplp/AmendolaRT19}.
The language of \ASPQ expands ASP with quantifiers over answer sets of ASP programs and allows the programmer to use the standard and natural programming methodology, known as generate-define-test~\cite{Lifschitz02}, to encode also problems beyond NP.

\myparagraph{Motivation.}
As in the case of ASP, the adoption of \ASPQ as a tool for modeling concrete problems~\cite{DBLP:conf/padl/FaberMC22,DBLP:conf/lpnmr/0001M22} has begun after the introduction of the first solver for \ASPQ, called \qasp~\cite{DBLP:conf/lpnmr/AmendolaCRT22}. 
\qasp was based on a translation to Quantified Boolean Formulae (QBF) with the aim of exploiting the well-developed and mature QBF-solving technology~\cite{DBLP:journals/ai/PulinaS19}. 
However, the implementation of the QBF encoding employed in \qasp is very general and might produce formulas that are hard to evaluate for existing QBF solvers because of the large number of symbols and sub-clauses. 
Moreover, Amendola et al.~\shortcite{DBLP:conf/lpnmr/AmendolaCRT22} observed that the implementation of the translation procedure could --in some specific cases-- be so memory-hungry to prevent the production of the QBF formula even when a considerable amount of memory is available. 
Moreover, \qasp's performance is dependent on the choice of the back-end QBF solver that performs differently over different domains.
This means that \textit{the quest for techniques resulting in faster solvers for \ASPQ is still open and challenging}, and directly impacts the deployment of \ASPQ applications. 

\myparagraph{Contributions.}
In this paper, we present a new implementation of \ASPQ that builds on the ideas of \qasp, but features both a more efficient encoding procedure and new optimized encodings of \ASPQ programs in QBF. 
More specifically, we provide:
\begin{enumerate}
\item An approach that exploits the well-founded semantics~\cite{DBLP:journals/jacm/GelderRS91} for simplifying \ASPQ programs.
\item The identification of a natural syntactic fragment of \ASPQ programs that can be directly translated to a QBF in Conjunctive Normal Form (CNF), thereby avoiding costly normalization steps. 
\item A new system for \ASPQ implemented in Python, called \pyqasp, 
that is modular and features an automatic selection of a suitable back-end for the given input.
\end{enumerate}

The new translations in QBF can produce smaller (or equally large) formulas, in terms of the number of quantifiers, variables, and clauses, with respect to the ones employed in \qasp, and this results in a more efficient evaluation process. 
A porting to the \ASPQ setting of the algorithm selection methodology employed by the competition-winning solver, \texttt{ME-ASP}~\cite{DBLP:journals/tplp/MarateaPR14}, allows \pyqasp to deliver steady performance over different problem domains.
An experimental analysis shows when the new optimizations provide benefits and demonstrates empirically that \pyqasp outperforms \qasp, and compares favorably with the implementation of stable-unstable semantics by Janhunen~\shortcite{DBLP:conf/padl/Janhunen22}.


\section{Preliminaries}\label{sec:pre}
In this section, we provide preliminary notions concerning logic programs, answer sets, well-founded semantics, and \ASPQ.
For ease of presentation, we focus our attention on propositional logic programs, but our methods are applicable to the full language of ASP.

\subsection{Programs}\label{sec:pre:prog}
An atom is a propositional variable, and a literal is either an atom $a$ or its negation $\sim a$ where $\sim$ represents negation as failure. A literal $l$ is positive (resp. negative) if it is of the form $a$ (resp. $\sim a$). The complement of a literal $l$, $\overline{l}$, is $\sim a$ if $l=a$ or $a$ if $l = \sim a$. Given a set of literals $L$, $\neg L$ denotes the set of literals $\{ \overline{l} \mid l \in L\}$. A choice atom is an expression of the form $\{a_1;\cdots;a_m\}$ where $m\ge 0$ and $a_1,\ldots,a_m$ are atoms. A rule is an expression of the form $h \leftarrow l_1, \ldots, l_n$ where $n \ge 0$, $h$ is either an atom or a choice atom, referred to as rule head and denoted by $H_r$, and $l_1,\ldots,l_n$ is a conjunction of literals, referred to as body and denoted by $B_r$. A rule is a fact if it has an empty body; it is normal if $h$ is an atom; it is a choice rule if $h$ is a choice atom. A constraint is a rule with an empty head $\leftarrow l_1, \ldots l_n$, which is a shorthand for $x \leftarrow l_1, \ldots, l_n, \sim x$ with $x$ being a fresh atom not occurring anywhere else. 
A choice rule $\{a_1;\cdots;a_m\} \leftarrow $ is (for simplicity) a shorthand for the rules
$a_i \leftarrow \sim na_i$, $na_i \leftarrow \sim a_i$ 
for $1 \le i \le m$, where all $na_i$ are fresh atoms not appearing elsewhere.
%
A program is a finite set of normal rules.

Given a program $P$, the dependency graph, $G_P$, is a directed labeled graph where nodes are atoms in $P$ and there is a positive (resp. negative) arc $(a_1,a_2)$ if there exists a rule $r \in P$ such that $a_1$ (resp. $\sim a_1$) appears in the body of $r$ and $a_2$ is the head of $r$. $P$ is stratified if $G_P$ contains no cycles involving negative arcs.

\myparagraph{Stable Models Semantics.} \label{sec:pre:sm}
Given a program $P$, the Herbrand Base, $\mathcal{B}_P$, is the set of atoms occurring in $P$. 
A (partial) interpretation $I$ is a subset of $\mathcal{B}_P \cup \neg \mathcal{B}_P$. 
A literal $l$ is true (resp. false) w.r.t.\ $I$ if $l \in I$ (resp. $\overline{l} \in I$), otherwise it is undefined. 
A conjunction of literals is true w.r.t.\ $I$ if all literals are true. 
An interpretation $I$ is consistent if for each $l \in I$, $\overline{l} \notin I$; 
it is total if for each $a \in \mathcal{B}_P$, $a$ is either true or false w.r.t. $I$. 
A rule $r$ is falsified w.r.t. $I$ if $B_r$ is true and $H_r$ is false.
A rule $r$ is satisfied w.r.t. $I$ if $B_r$ is false or $H_r$ is true.
A consistent interpretation $I$ is a model of $P$ if it does not falsify any rule in $P$.

A total model $M$ is a (subset-)minimal model if does not exist a total model $M_1$ such that $M_1^+ \subset M^+$.
Given a model $M$, the (Gelfond-Lifschitz) reduct of $P$ w.r.t. $M$, $P^M$, is obtained from $P$ by removing rules with negative literals in the body that are false w.r.t. $M$ and deleting from the body of the remaining rules all negative literals that are true w.r.t. $M$. A total model $M$ is an answer set of $P$ if $M$ is a minimal model 
of $P^M$~\cite{DBLP:journals/ngc/GelfondL91}.

$\as(P)$ is the set of answer sets of $P$. 
P is \textit{coherent} iff $\as(P) \neq \emptyset$.

\myparagraph{Well-founded Semantics.}\label{sec:pre:wf}
Let $P$ be a program and $I$ be an interpretation, a set of atoms $U \subseteq \mathcal{B}_P$ is an unfounded set of $P$ w.r.t. $I$ if for each rule $r \in P$ such that $H_r \in U$, $B_r$ is false w.r.t. $I$ or $B_r \cap U \neq \emptyset$. The greatest unfounded set of $P$ w.r.t. $I$, $U_P(I)$, is defined as the union of all unfounded sets of $P$ w.r.t. $I$. Let $T_P(I)$ be the set of atoms $a \in \mathcal{B}_P$ such that there exists a rule $r \in P$ having $a$ in the head and a true body w.r.t. $I$, the well-founded operator, $\mathcal{W}_P(I)$, is defined as $T_P(I) \cup \neg U_P(I)$. The (partial) well-founded model is defined as the least fixed point of the operator $\mathcal{W}_P$~\cite{DBLP:journals/jacm/GelderRS91}. 
The well-founded model $W$ of $P$ is a subset of each answer set of $P$.

\subsection{Answer Set Programming with Quantifiers}\label{sec:pre:aspq}

An \textit{ASP with Quantifiers} (\ASPQ) program $\Pi$ is
of the form~\cite{DBLP:journals/tplp/AmendolaRT19}:%
\begin{equation}
\Box_1 P_1\ \Box_2 P_2\ \cdots\ \Box_n P_n :  C ,
\label{eq:qasp}
\end{equation}

\noindent where, for each $i=1,\ldots,n$, $\Box_i \in \{ \exists^{st}, \forall^{st}\}$, $P_i$ is an ASP program, 
and $C$ is a stratified ASP program possibly with constraints.
An  \ASPQ program $\Pi$ of the form (\ref{eq:qasp}) is \emph{existential}  if $\Box_1 =\exists^{st}$, otherwise it is \emph{universal}.

Given a logic program $P$, a total interpretation $I$ over the Herbrand base $\mathcal{B}_P$, and an \ASPQ  
program $\Pi$ of the form (\ref{eq:qasp}), we denote by $\fix_P(I)$ the set of 
facts and constraints          
$\{ a \mid a\in I \cap \mathcal{B}_P\} \cup \{ \leftarrow a \mid a\in \mathcal{B}_P \setminus I \}$,
and by $\Pi_{P,I}$ the \ASPQ\ program of the form~(\ref{eq:qasp}), where $P_1$ is 
replaced by $P_1\cup \fix_P(I)$, that is, 
$\Pi_{P,I} =\Box_1 (P_1\cup \fix_P(I))\ \Box_2 P_2\  \cdots \Box_n P_n :  C.$

The \textit{coherence} of \ASPQ programs is defined by induction as follows:
\begin{itemize}
\item $\exists^{st} P:C$ is coherent, if there exists $M\in AS(P)$ such that 
$C\cup \fix_P(M)$ is coherent;
\item $\forall^{st} P:C$ is coherent, if for every $M\in AS(P)$, 
$C\cup \fix_P(M)$ is coherent;
\item $\exists^{st} P\ \Pi$ is coherent, if there exists $M\in AS(P)$ such 
that $\Pi_{P,M}$ is coherent;
\item $\forall^{st} P\ \Pi$ is coherent, if for every $M\in AS(P)$, $\Pi_{P,M}$
is coherent.
\end{itemize}
``Unwinding'' the definition for a quantified program  
$\Pi=\exists^{st} P_1 \forall^{st} P_2 \cdots \exists^{st} P_{n-1} \forall^{st} P_n: C$
yields that $\Pi$
is coherent if there exists an answer set $M_1$ of $P_1'$ 
such that for each answer set $M_2$ of $P_2'$ 
there is an answer set $M_3$ of $P_3', \ldots,$ 
there is an answer set $M_{n-1}$ of $P_{n-1}'$ 
such that for each answer set $M_n$ of $P_n'$, 
there is an answer set of $C\cup \fix_{P_n'}(M_n)$, 
where $P_1'=P_1$, and $P_i'=P_i\cup \fix_{P_{i-1}'}(M_{i-1})$, if $i\geq2$. 
For an existential \ASPQ program $\Pi$, $M \in AS(P_1)$ is a \textit{quantified answer set} of $\Pi$, if $(\Box_2 P_2 \cdots \Box_n P_n :  C)_{P_1,M}$ is coherent.
We denote by $QAS(\Pi)$ the set of all quantified answer sets of $\Pi$.

\nop{For example,  given a 2-QBF formula $\Phi=\exists X\forall Y G$, 
where $G=D_1 \lor\ldots \lor D_h$ is a DNF, and
$D_i= L_{i,1} \land\ldots \land L_{i,k_i}$ and $L_{i,j}$ are literals 
over $X\cup Y$, we encode $\Phi$ in a quantified \ASPQ program $\Pi_\Phi=\exists^{st} P_1 \forall^{st} P_2: C$ where 
$$P_1 = \{ \{x_1,\dots,x_n\} \},\quad P_2 = \{ \{y_1,\dots,y_m\} \},$$
$$C = \{ sat \leftarrow \sigma(L_{i,1}),\ldots,\sigma(L_{i,k_i}) \mid \forall i=1,\ldots,h\} $$$$\cup \{ \leftarrow not\ sat \}.$$
A fresh atom $sat$ models satisfiability, and  
set $\sigma(z)=z$ and $\sigma(\neg z)=not\ z$.
Here, a satisfiability of an existential 2-QBF is encoded directly.
Indeed $P_1$ (which consists of a single choice rule) guesses an 
assignment to $X$ s.t. for all assignments to $Y$ generated by $P_2$ (which, as $P_1$, consists of a single choice rule), 
$sat$ must be derived by satisfying at least one conjunct in $\varphi$, i.e., 
$\Pi_\Phi$ is satisfiable iff $\Phi$ is. The example can be easily generalized to general QBFs. 
}

Given a set of propositional atoms $A$, we denote by $ch(A)$ the program $\{ \{ a \} | a \in A \}$ made of  choice rules over atoms in $A$.  
For two ASP programs $P$ and $P'$, let $\Int(P,P')$ be the set $\mathcal{B}_P \cap \mathcal{B}_{P'}$ of common atoms. 
For two programs $P$ and $P'$, the choice interface program $\CH(P,P')$ is defined as $ch(\Int(P,P'))$. 
For a propositional formula $\Phi$, $var(\Phi)$ denotes the variables occurring in $\Phi$.
For an \ASPQ program $\Pi$, and an integer $1\leq i \leq n$, we define the program $P_i^{\leq}$ as the union of program $P_j$ with $1 \leq j \leq i$. 
Given an input program $\Pi$ of the form (\ref{eq:qasp}), the intermediate versions $G_i$ of its subprograms, and the QBF $\Phi(\Pi)$ encoding $\Pi$ are:
\[
G_i = \left\{\begin{array}{ll}
 P_1 & i=1 \\
 P_i \cup \CH(P_{i-1}^{\leq},P_i) & 1<i\leq n\\
 C \cup \CH(P_n^{\leq},C) & i=n+1
\end{array}\right.
\]
$$\Phi(\Pi) = \boxplus_1 \cdots \boxplus_{n+1} \left(\bigwedge_{i=1}^{n+1} (\phi_i \leftrightarrow \CNF(G_i))\right) \wedge \phi_c $$ 

\noindent where $\CNF(P)$ is a CNF formula encoding the program $P$ (such that models of $\CNF(P)$ correspond to $\as(P)$); 
$\phi_1, \ldots, \phi_{n+1}$ are fresh propositional variables; 
$\boxplus_{i} = \exists x_i$ if $\Box_i=\exists^{st}$ or $i=n+1$, and $\boxplus_{i} = \forall x_i$ otherwise, where $x_i = var(\phi_i \leftrightarrow \CNF(G_i))$ for $i=1,\cdots,n+1$, and $\phi_c$ is the formula 
$$\phi_c = \phi_1' \odot_1 ( \phi_2'  \odot_2 ( \cdots \phi_n' \odot_n (\phi_{n+1}) \cdots ))$$
\noindent where $\odot_i = \vee$ if $\Box_i=\forall^{st}$, and $\odot_i = \wedge$ otherwise, and $\phi_i' = \neg \phi_i$ if $\Box_i=\forall^{st}$, and $\phi_i' = \phi_i$ otherwise.
Intuitively, there is a direct correspondence between the quantifiers in $\Pi$ and $\Phi(\Pi)$; moreover, in each subprogram of $\Pi$ (i.e, $P_1, \cdots, P_n, C$) the atoms interfacing  with preceding subprograms are left open; then the programs are converted into equivalent CNF formulas; finally, the formula $\phi_c$ is built to constrain the variable assignments corresponding to the stable models of each subprogram so that they behave as required by the semantics of \ASPQ. 

\begin{thm}[Amendola et al.~\citeyearNP{DBLP:conf/lpnmr/AmendolaCRT22}]\label{th:original-translation}
Let $\Pi$ be a quantified program. Then $\Phi(\Pi)$ is true iff $\Pi$ is coherent.
\end{thm}
%

\section{Simplification based on well-founded semantics}\label{sec:wf}
We present an alternative approach that exploits the well-founded semantics in order to obtain a simplified but equivalent \ASPQ program that allows a more compact translation into a QBF formula both in terms of number of clauses and average clause length.
\begin{definition}
\label{def:residual}
Given a program $P$ and its well-founded model $\mathcal{W}$, the \textit{residual program}, $R(P)$,
is obtained from $P$ by removing all those rules with a false body w.r.t. $\mathcal{W}$ and true literals in $\mathcal{W}$ from the bodies of the remaining ones.
\end{definition}
\begin{prop}
\label{prop:eq:residual}
Given a program $P$ and its well-founded model $\mathcal{W}$, $\as(P) = \as(R(P))$
\end{prop}
Note that $\mathcal{W}$ is a subset of any stable model $M$ of $P$ and so the missing rules in $R(P)$ have a false body w.r.t. M and so they are trivially satisfied by $M$. Each rule in $P$ that has not been removed in $R(P)$ is satisfied if and only if the simplified rule in $R(P)$ is satisfied. 
So, it can be proved that the reduct of the two programs have the same minimal models for every model $M$.
Let $P$ and $P^{\prime}$ be two programs and $\mathcal{W}$ the well-founded model of $P$, $\CH^{\prime}(P,P^{\prime}) = \{\{a\} \mid a \in (Int(P,P^{\prime})\setminus(\mathcal{W} \cup \neg \mathcal{W}))\} \cup \{a \leftarrow \ \mid a \in (Int(P,P^{\prime})\cap \mathcal{W})\}$.
Given an \ASPQ program $\Pi$ of the form (\ref{eq:qasp}), the QBF encoding $\Phi^{WF}(\Pi)$ is as follows:
\[
    G^{WF}_i = \left
    \{\begin{array}{cc}
    R(P_1) & i=1 \\
    R(P_i \cup \CH^{\prime}(P_{i-1}^{\leq},P_i)) & i\in [2..n]\\
    R(C \cup \CH^{\prime}(P_n^{\leq},C)) & i=n+1
    \end{array}\right.
\]
$$\Phi^{\mathcal{WF}}(\Pi) = \boxplus_1 \cdots \boxplus_{n+1} \left(\bigwedge_{i=1}^{n+1} (\phi^{\mathcal{WF}}_i \leftrightarrow \CNF(G^{WF}_i)\right) \wedge \phi_c,$$
\noindent where $\CNF(G^{WF}_i)$ is a CNF formula encoding $G^{WF}_i$, 
$\phi^{\mathcal{WF}}_1, \ldots, \phi^{\mathcal{WF}}_{n+1}$ are fresh propositional variables; 
$\boxplus_{i} = \exists x_i$ if $\Box_i=\exists^{st}$ or $i=n+1$, and $\boxplus_{i} = \forall x_i$ otherwise, where $x_i = var(\phi_i^{\mathcal{WF}} \leftrightarrow \CNF(G^{WF}_i))$ for $i=1,\cdots,n+1$, and $\phi_c$ is the formula 
$$\phi_c = \phi_1' \odot_1 ( \phi_2'  \odot_2 ( \cdots \phi_n' \odot_n (\phi_{n+1}) \cdots ))$$
\noindent where $\odot_i = \vee$ if $\Box_i=\forall^{st}$, and $\odot_i = \wedge$ otherwise, and $\phi_i' = \neg \phi^{\mathcal{WF}}_i$ if $\Box_i=\forall^{st}$, and $\phi_i' = \phi^{\mathcal{WF}}_i$ otherwise.
Intuitively, $\Phi^{\mathcal{WF}}(\Pi)$ is constructed by following the encoding proposed in Section~\ref{sec:pre:aspq} 
but each program $P_i$ is replaced by its residual w.r.t.\ the well-founded model. 
\newcommand{\thmtextphiwf}{
Let $\Pi$ be an \ASPQ program, then $\Phi^{\mathcal{WF}}(\Pi)$ is true iff $\Pi$ is coherent.}
\begin{thm}
\label{thm:phi:wf}
\thmtextphiwf{}
\end{thm}

The programs $G^{WF}_i$ preserve the coherence of $\Pi$ due to their construction and Proposition~\ref{prop:eq:residual}. Together with Theorem~\ref{th:original-translation}, the result follows.

\begin{prop}
\label{prop:unsat:level}
Let $\Pi$ be an \ASPQ program of the form (\ref{eq:qasp}), if $G^{WF}_i$ is incoherent then $\phi^{\mathcal{WF}}_i$ can be replaced by $\bot$. 
\end{prop}
It is easy to see that if $G^{WF}_i$ is incoherent then $\CNF(G^{WF}_i)$ is unsatisfiable and so $\phi^{\mathcal{WF}}_i$ can be replaced by $\bot$. 
\begin{example}
Let $P_i$ be the program $\{a \leftarrow a; \quad p \leftarrow \ \sim a, \ \sim p\}$.
Since predicates occurring in $P_i$ are $p$ and $a$ and both are defined at level $i$ then $\CH^{\prime}(P^{\leq}_{i-1},P_i) = \emptyset$ and so, $G^{WF}_i = R(P_i)$. Since the well-founed model of $P_i$ is $\mathcal{W}=\{\sim a\}$, then $R(P_i) = p\leftarrow \sim p$  that is incoherent, and so $\CNF(G^{WF}_i) = p \wedge \neg p$ is unsatisfiable.
\end{example}
\begin{prop}
\label{prop:stop:encoding}
 Let $\Pi$ be an \ASPQ program of the form (\ref{eq:qasp}), if $G^{WF}_k$ is incoherent then $$\Phi^{\mathcal{WF}}(\Pi) \equiv \boxplus_1 \cdots \boxplus_{k-1} \left(\bigwedge_{i=1}^{k-1} (\phi^{\mathcal{WF}}_i \leftrightarrow \CNF(G^{WF}_i)\right) \wedge \phi_c',$$
\noindent where 
$\phi^{\mathcal{WF}}_1, \ldots, \phi^{\mathcal{WF}}_{k-1}$ are fresh propositional variables; 
$\boxplus_{i} = \exists x_i$ if $\Box_i=\exists^{st}$, and $\boxplus_{i} = \forall x_i$ otherwise, $x_i = var(\phi_i \leftrightarrow \CNF(G^{WF}_i))$ for $i=1,\cdots,k-1$, and $\phi_c^{\prime}$ is the formula 
$$\phi_c' = \phi_1' \odot_1 ( \phi_2'  \odot_2 ( \cdots \phi_{k-1}' \odot_{k-1} (\phi_{k}') \cdots ))$$
\noindent where $\odot_i = \vee$ if $\Box_i=\forall^{st}$, and $\odot_i = \wedge$ otherwise, for $i\in [1,\cdots,k-1]$, $\phi_i' = \neg \phi^{\mathcal{WF}}_i$ if $\Box_i=\forall^{st}$, and $\phi_i' = \phi^{\mathcal{WF}}_i$ otherwise, and $\phi_k' = \top$ if $\Box_i=\forall^{st}$, and $\phi_k' = \bot$ otherwise.
\end{prop}
From Proposition~\ref{prop:stop:encoding} it follows that if $k=1$ then $\Phi^{\mathcal{WF}}(\Pi)=\phi_k'$ where $\phi_k' = \top$ if $\Box_i=\forall^{st}$, and $\phi_k' = \bot$ otherwise. So, in such cases we can determine the coherence of the \ASPQ directly in the encoding phase. 
%
\begin{prop}
Given an \ASPQ program $\Pi$, it holds that $|clauses(\Phi^{\mathcal{WF}}(\Pi))|\leq |clauses(\Phi(\Pi))|$
\end{prop}
We observe that, by definition~\ref{def:residual}, residual subprograms are obtained by removing some trivially satisfied rules in every stable model or deleting literals from the rules' body by means of the well-founded operator. 
This results in a smaller CNF both in terms of the number of clauses, since potentially fewer rules are encoded, and also in average clause length, since each rule is transformed into one or more clauses that have fewer literals. 
Moreover, by propagating information from the well-founded model of previous levels, stable models of the following levels are restricted to those that are coherent with previous models, if any. If no models exist, then the resulting QBF formula is pruned at the incoherent level. In the worst case scenario, that is $\mathcal{W}=\emptyset$ for every program, $\CH^{\prime}$ produces the same interface program produced by $\CH$, $G_i = G^{WF}_i$ and so $\Phi^{\mathcal{WF}}(\Pi) = \Phi(\Pi)$.

\section{Direct CNF encodings for \ASPQ{} programs}\label{sec:newenc}

Formulas $\Phi(\Pi)$ and $\Phi^{\mathcal{WF}}(\Pi)$ are not in CNF because of the presence of equivalences for each $i$  and the final formula $\phi_c$ (which is not in CNF either). While this might be seen as a minor issue, the translation of non-CNF formulas into CNF by means of a Tseytin transformation can be a time-consuming procedure that increases the length of the formulas and introduces extra symbols that could slow down QBF solvers. 

A natural question, therefore, is whether it is possible to identify classes of \ASPQ programs such that the resulting QBF formula is in CNF. In the following, we can answer this positively and provide some conditions under which this is possible. 

Given a program $P$, $heads(P)$ denotes the set of atoms that appear in the head of some rules in $P$, $\textit{facts}(P)$ denotes the set of facts in $P$.  
Given an \ASPQ program $\Pi$, $Ext_i = heads(P_i) \cap  \bigcup_{j>i} Int(P_i,P_j)$ denotes the set of atoms defined in $P_i$ that belong to the interface of the following levels.
\begin{definition}
\label{def:trivial:subprogram}
    Let $\Pi$ be an \ASPQ program, a subprogram $P_i$ is \textit{trivial} if the following conditions hold: $(i)$ $\forall \ 1\leq j<i: Int(P_i,P_j) \subseteq \textit{facts}(P_j)$ 
    and $(ii)$ $\as(P_i)\mid_{Ext_i} = 2^{Ext_i}$, 
    where $\as(P_i)\mid_{Ext_i} = \{ S\cap Ext_i \mid S \in \as(P_i)\}$ and $2^{Ext_i}$ denotes the power set of $Ext_i$.
\end{definition}
Let $\Pi$ be an \ASPQ program, $K=\{k \mid P_k \textit{ is a trivial subprogram} \wedge k\leq n\}$, the QBF encoding $\Phi^K(\Pi)$ is defined as follows:
$$\Phi^{K}(\Pi) = \boxplus_1 \cdots \boxplus_{n+1} \left(\bigwedge_{\substack{ i=1\\i\notin K}}^{n+1} (\phi_i \leftrightarrow \CNF(G_i))\right) \wedge \phi^K_c,$$  
$\boxplus_{i} = \exists x_i$ if $\Box_i=\exists^{st}$ or $i=n+1$, and $\boxplus_{i} = \forall x_i$ otherwise, $x_i = var(\phi_i \leftrightarrow \CNF(G_i))$ if $i\notin K$, otherwise $x_i = Ext_i$, and $\phi^K_c$ is 
$\phi^K_c = \phi_{i_1}' \odot_{i_1}(\phi_{i_2}' \odot_{i_2} ( \cdots (\phi_{i_m}' \odot_{i_m} (\phi_{n+1}) ) \cdots) )$
\noindent where $E=\{1,\cdots,n\}\setminus K = \{i_1,i_2,\cdots,i_m\}$, $i_1<i_2<\cdots<i_m$, $\odot_i = \vee$ if $\Box_i=\forall^{st}$, and $\odot_i = \wedge$ otherwise, and $\phi_i' = \neg \phi_i$ if $\Box_i=\forall^{st}$, and $\phi_i' = \phi_i$ otherwise, with $i \in E$.    
\newcommand{\thmtextomittedforall}{
Let $\Pi$ be an \ASPQ program, and $K=\{k \mid P_k \textit{ is a trivial subprogram} \wedge k\leq n\}$, then $\Phi^K(\Pi)$ is satisfiable iff $\Pi$ is coherent.}
\begin{thm}
\label{thm:omitted:forall}
\thmtextomittedforall{}
\end{thm}
For any trivial $P_k$ the formula $CNF(G_k)$ is a tautology, allowing for the simplifications that result in $\Phi^K(\Pi)$.

\begin{prop}
\label{prop:trivial:forall}
Let $\Pi$ be an \ASPQ program, $K=\{k \mid P_k \textit{ is a trivial subprogram} \wedge k\leq n\}$. If for each subprogram $P_i$ such that $\Box_i = \forall^{st}$, it holds that $i \in K$, then $\Phi^K$ is equivalent to the CNF formula with the same quantifiers: $\Phi^K_{CNF} = \boxplus_1 \cdots \boxplus_{n+1} \bigwedge_{j\in J} \CNF(G_j)$, where $J = \{1,\ldots,n+1\}\setminus K$.
\end{prop}

Programs satisfying Proposition~\ref{prop:trivial:forall} have a direct CNF encoding. However, verifying that a program is trivial is hard since  Definition \ref{def:trivial:subprogram} requires a co-NP check. 
There is, however, a very common syntactic class of programs for which this property is trivially satisfied: the \ASPQ programs of the form (\ref{eq:qasp}) where each $P_i$ contains only choice rules. 
An example is the encoding of QBF in \ASPQ proposed by Amendola et al. \shortcite{DBLP:journals/tplp/AmendolaRT19}.
%
In the following, we identify a larger class of \ASPQ programs featuring a direct encoding in CNF, the ones that follow the well-known Guess and Check methodology~\cite{DBLP:journals/amai/EiterG95}. 

\begin{definition}
\label{def:guess:check}
    An ASP program $P$ is \textit{Guess\&Check} if it can be partitioned into two subprograms $G_P$, Guess, $C_P$, Check, where $G_P$ contains only choice rules and $C_P$ is the maximal stratified subprogram possibly with constraints of $P$, such that
    $\{ H_r \mid r \in C_P\} \cap \mathcal{B}_{G_P} = \emptyset$.
\end{definition}
\begin{example}
    Let $P$ be the program $\{
         r_1:\{a;b;c\} \leftarrow, \ 
         r_2:d \leftarrow a,\ 
         r_3:d \leftarrow b,\ 
         r_4:\leftarrow c,d.\}$.
$P$ can be partitioned in $G_P=\{r_1\}$ and $C_P=\{r_2,r_3,r_4\}$.  
\end{example}

Guess\&Check programs feature a modularity property.

\begin{prop}
\label{prop:stable:guess:check}
    Let $P$ be a Guess\&Check program then $M \in \as(P)$ iff there exists $M^{\prime} \in \as(G_P)$ such that $M=M^{\prime}\cup W$ and $W \in \as(C_P \cup fix_{G_P}(M^{\prime}))$.
\end{prop}


\begin{definition}\label{def:aspq:gc}
An \ASPQ program $\Pi$ of the form~(\ref{eq:qasp}) is \textit{Guess\&Check} if $(i)$ universal and existential quantifiers are alternated, and $(ii)$ all $P_i$ with $\Box_i = \forall^{st}$ are Guess\&Check subprograms.
\end{definition}

The following definition provides a rewriting for a universal Guess\&Check subprogram.

\begin{definition}
Given a Guess\&Check program $P_1$, a program $P_2$, and a propositional atom $u$ such that $u \notin (\mathcal{B}_{P_1}\cup \mathcal{B}_{P_2})$, we define

\[
    \tau(u,P_1) = \left
    \{\begin{array}{ll}
         H_r \leftarrow B_r &  r \in C_{P_1} \wedge H_r \neq \emptyset\\
         u \leftarrow B_r &  r \in C_{P_1} \wedge H_r = \emptyset\\
    \end{array}\right.
\]
$$\sigma(u,P_1,P_2) = \tau(u,P_1) \cup \rho(u,P_2)$$
$$\rho(u,P_2) = \{H_r\leftarrow B_r,\sim u \mid r \in P_2\}$$

Given a \textit{Guess\&Check} \ASPQ program $\Pi$, let $i\in [1,\dots, n]$ be such that $\Box_i = \forall^{st}$:
\[
    \Pi^{GC_i} = \left
    \{\begin{array}{ll}
         \Box_1 P_1 \cdots \forall^{st} G_{P_i} :  \sigma(u,P_i,C) &   i=n\\
         \Box_1 P_1 \cdots \forall^{st} G_{P_i} \exists^{st} \sigma(u,P_i,P_{i+1}) :  \rho(u,C)  &  i=n-1\\
         \Box_1 P_1 \cdots \forall^{st} G_{P_i} \exists^{st} \sigma(u,P_i,P_{i+1}) \forall^{st} P_{i+2} \cup \{\leftarrow u\}\cdots \Box_n P_n :  C &  otherwise\\
    \end{array}\right.
\]
\end{definition}
\newcommand{\thmtextmovedequiv}{
Let $\Pi$ be a Guess\&Check \ASPQ program, for each
$i\in [1,\dots, n]$ such that $\Box_i = \forall^{st}$, $\Pi$ is coherent iff $\Pi^{GC_i}$ is coherent.}

\begin{thm}
\label{thm:moved:equiv}
\thmtextmovedequiv{}
\end{thm}
This theorem holds because in $\Pi^{GC_i}$ the answer sets of the replaced subprograms are preserved with respect to those in $\Pi$. Interpretations that violate constraints become additional answer sets, that are either invalidated in the next universal subprogram or do not affect the coherence of $\Pi$.

We now define a recursive transformation that, given a Guess\&Check \ASPQ program $\Pi$, builds a sequence of \ASPQ programs ($\Pi_1,\dots, \Pi_n$) such that the last program of that sequence is both equivalent to $\Pi$ and features an encoding in CNF.

\begin{definition}
Let $\Pi$ be a Guess\&Check \ASPQ program, then
\[
\Pi_i = \left\{\begin{array}{ll}
 \Pi & i=1 \wedge \Box_i = \exists^{st} \\
 \Pi^{GC_1} & i=1 \wedge \Box_i = \forall^{st} \\
 \Pi_{i-1} & i\in [2..n] \wedge \Box_i = \exists^{st}\\
 (\Pi_{i-1})^{GC_i} & i\in [2..n] \wedge \Box_i = \forall^{st}
\end{array}\right.
\]
\end{definition}

\newcommand{\thmtextallGC}{
Let $\Pi$ be a Guess\&Check \ASPQ program, and $K$ be the set of indexes $K=\{k | k\in [1,\dots, n] \wedge \Box_k=\forall^{st}\}$ (i.e., s.t. $P_k$ a universally quantified subprogram), then $\Pi$ is coherent iff $\Phi^K_{CNF}(\Pi_n)$ is satisfied.}

\begin{thm}\label{thm:allGC}
\thmtextallGC{}
\end{thm}

Here all universal subprograms are replaced by programs that contain only choice rules and are therefore trivial. The result then follows from Theorem~\ref{thm:omitted:forall}, Proposition~\ref{prop:trivial:forall}, and Theorem \ref{thm:moved:equiv}.

\section{Implementation and Experiments}\label{sec:experiments}\label{sec:system}
In this section, we describe our implementation and discuss an experimental analysis conducted to $(i)$ demonstrate empirically the efficacy of the techniques described above, $(ii)$ compare \pyqasp with \qasp, and $(iii)$ compare \pyqasp with a recent implementation of the stable unstable semantics~\cite{DBLP:conf/padl/Janhunen22}.

\subsection{Implementation, Benchmarks and Experiment Setup}
The \pyqasp system is an implementation in Python of the transformation techniques described in Sections~\ref{sec:pre:aspq}, \ref{sec:wf}, and \ref{sec:newenc}.
The input \ASPQ{} program is transformed into a QBF formula to be processed later by a QBF solver that supports the QCIR format.
\pyqasp can handle non-propositional inputs, indeed the user can select either gringo~\cite{DBLP:conf/lpnmr/GebserKKS11} or iDLV~\cite{DBLP:journals/tplp/CalimeriDFPZ20} as grounders.
The SAT encoding of ASP subprograms ($CNF(\cdot)$) is produced using ASPTOOLS~\cite{DBLP:conf/ecai/Janhunen04,DBLP:journals/ki/Janhunen18a}.
The computation of the well-founded-based rewriting (see Section~\ref{sec:wf}) uses the  computation of the well-founded model in iDLV. 
\pyqasp is modular, in the sense that the user can choose a QBF solver to use as the back-end.
\pyqasp supports the same back-ends as \qasp, which are based on \textit{DepQBF}, \textit{Quabs}, and \textit{RareQS} QBF solvers.
Moreover, \pyqasp implements an automatic algorithm selection strategy, devised according to the methodology employed in the \texttt{ME-ASP} multi-engine ASP solver proposed by Maratea et al. \shortcite{DBLP:journals/tplp/MarateaPR14}, that selects automatically a suitable back-end for the given input.
A more detailed description of the evaluation process is available in \ref{app:implementation}. The source code is available at \url{https://github.com/MazzottaG/PyQASP.git}.

\myparagraph{\ASPQ Benchmarks.}
We run a suite of benchmarks that has already been used to assess the performance of \ASPQ implementations~\cite{DBLP:conf/lpnmr/AmendolaCRT22}. 
The suite contains encodings in \ASPQ and instances of four problems:
Quantified Boolean Formulas (QBF);
Argumentation Coherence (AC);
Minmax Clique (MMC);
Paracoherent ASP (PAR).
The suite comprises a selection of instances from QBF Lib (\url{https://www.qbflib.org/}), ICCMA 2019 (\url{http://argumentationcompetition.org/2019}), ASP Competitions~\cite{DBLP:journals/jair/GebserMR17}, and PAR instances by Amendola et al. \shortcite{DBLP:journals/ai/AmendolaDFR21}. 
A detailed description of these benchmarks was provided by Amendola et al. \shortcite{DBLP:conf/lpnmr/AmendolaCRT22}.

\myparagraph{Experiment Setup.}
All the experiments of this paper were run on a system with 2.30GHz Intel(R) Xeon(R) Gold 5118 CPU and 512GB of RAM with Ubuntu 20.04.2 LTS (GNU/Linux 5.4.0-137-generic x86\_64).
Execution time and memory were limited to 800 seconds (of CPU time, i.e., user+system) and 12 GB, respectively. Each system was limited to run in a single core. 

\subsection{Impact of the new techniques}

\myparagraph{Compared methods.}
We run three variants of \pyqasp, namely:
\begin{itemize}
    \item \pyqaspv{}{}: basic encoding with gringo as grounder;
    \item \pyqaspv{}{WF}: basic encoding with well-founded simplification (iDLV as grounder);
    \item \pyqaspv{}{WF+GC}: well-founded simplification and direct encoding in CNF (i.e., production of a CNF encoding for \textit{guess\&check} programs).
\end{itemize}

These variants were combined with the following three QBF back-ends:

\begin{itemize}
    \item $\mathit{RQS}$: \textit{qcir-conv.py} (by Klieber - \url{https://www.wklieber.com/ghostq/qcir-converter.html}) transforms QCIR to the GQ format of \textit{RareQS} solver (by Janota \url{http://sat.inesc-id.pt/~mikolas/sw/areqs}), that is called.
    \item $\mathit{DEPS}$: \textit{qcir-conv.py} and \textit{fmla} convert the formula from QCIR to QDIMACS, \textit{bloqqer} (by Biere et al. - \url{http://fmv.jku.at/bloqqer}) simplifies it, then the QBF solver \textit{DepQBF} (by Lonsin - \url{https://lonsing.github.io/depqbf}) is called.
    \item $\mathit{QBS}$: The QBF solver \textit{Quabs} (by Tentrup - \url{https://github.com/ltentrup/quabs}) is called, with no pre-processor.
\end{itemize}

All this amounts to running 9 variants of \pyqasp.
In our naming conventions, the selected back-end is identified by a superscript, and a subscript identifies the optimizations enabled. 
For example, \pyqaspv{DEPS}{} indicates \pyqasp with back-end $\mathit{DEPS}$, and \pyqaspv{DEPS}{WF+GC} indicates \pyqasp with $\mathit{DEPS}$ back-end and all optimizations enabled.

\begin{figure}
     \centering
     \begin{subfigure}[b]{0.49\textwidth}
         \centering 
         \includegraphics[width=\textwidth,trim=80 50 80 85,clip]{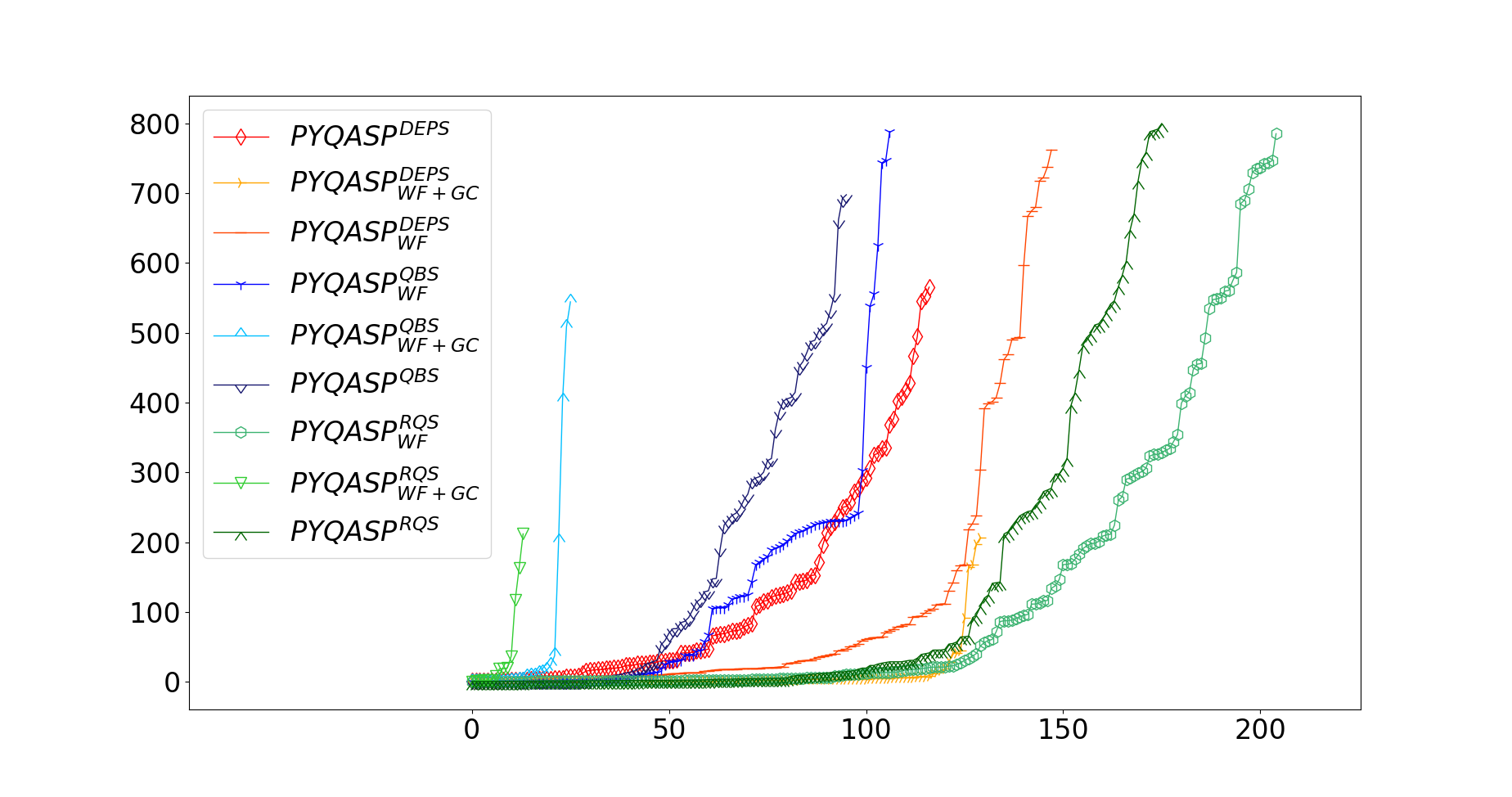}
         \caption{Argumentation Coherence (AC).}
         \label{fig:techniques:ac}
     \end{subfigure}
     \hfill
     \begin{subfigure}[b]{0.49\textwidth}
         \centering
         \includegraphics[width=\textwidth,trim=80 50 80 85,clip]{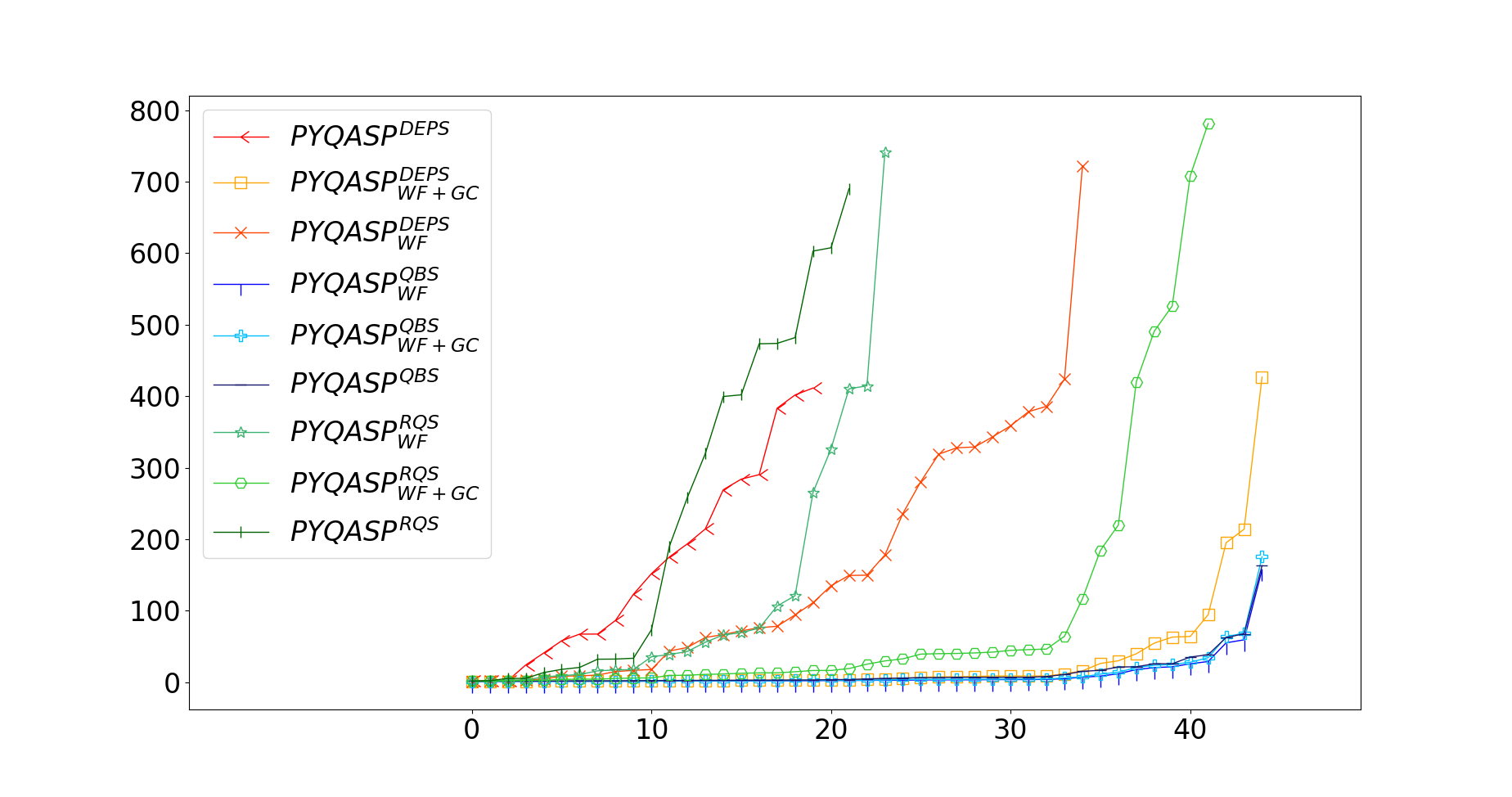}
         \caption{Min Max Clique (MMC).}
         \label{fig:techniques:mmc}
     \end{subfigure}
     \hfill
     \begin{subfigure}[b]{0.49\textwidth}
         \centering
         \includegraphics[width=\textwidth,trim=80 50 80 85,clip]{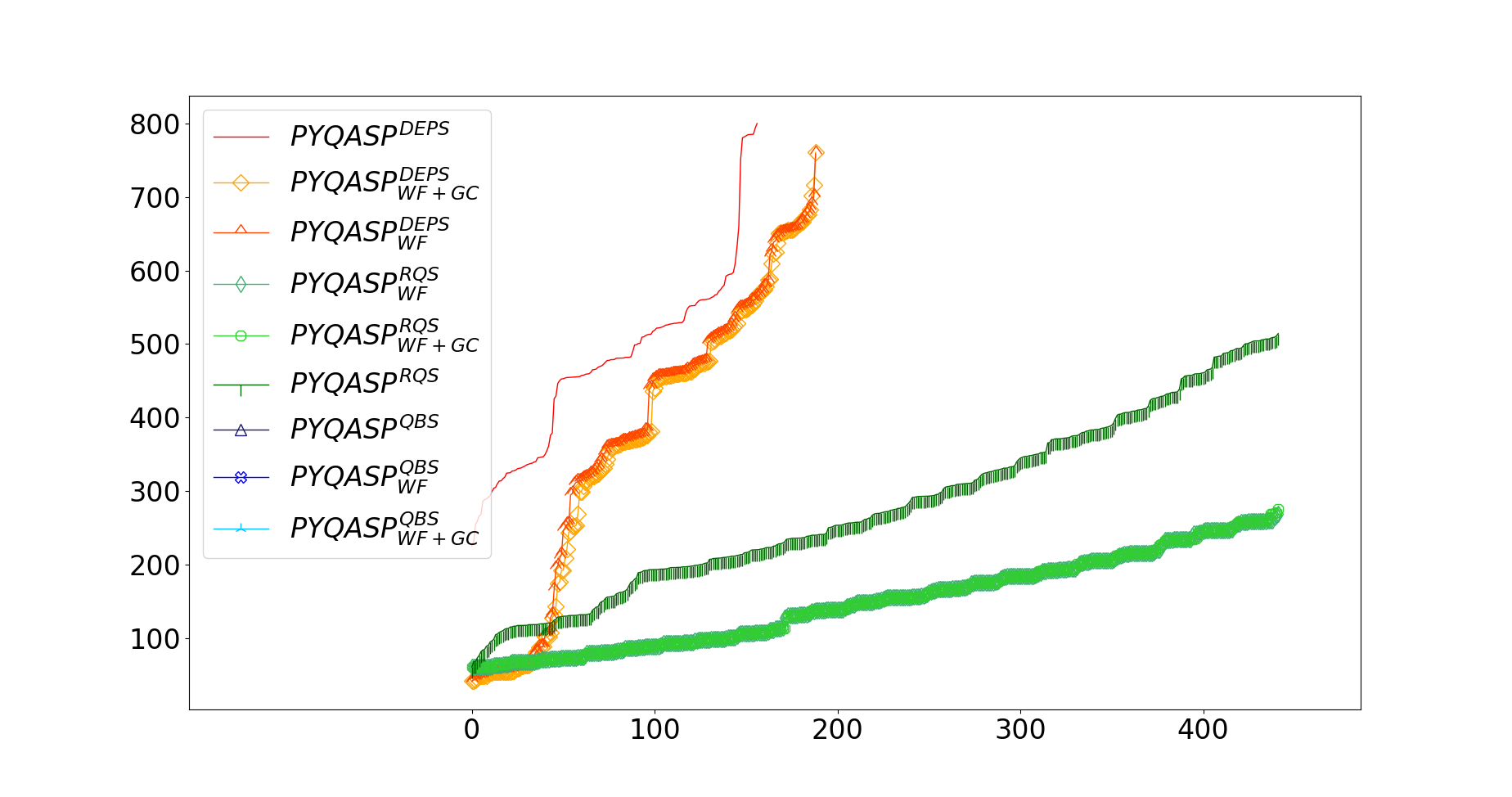}
         \caption{Paracoherent ASP (PAR).}
         \label{fig:techniques:par}
     \end{subfigure}
    \hfill
     \begin{subfigure}[b]{0.49\textwidth}
         \centering
         \includegraphics[width=\textwidth,trim=80 50 80 85,clip]{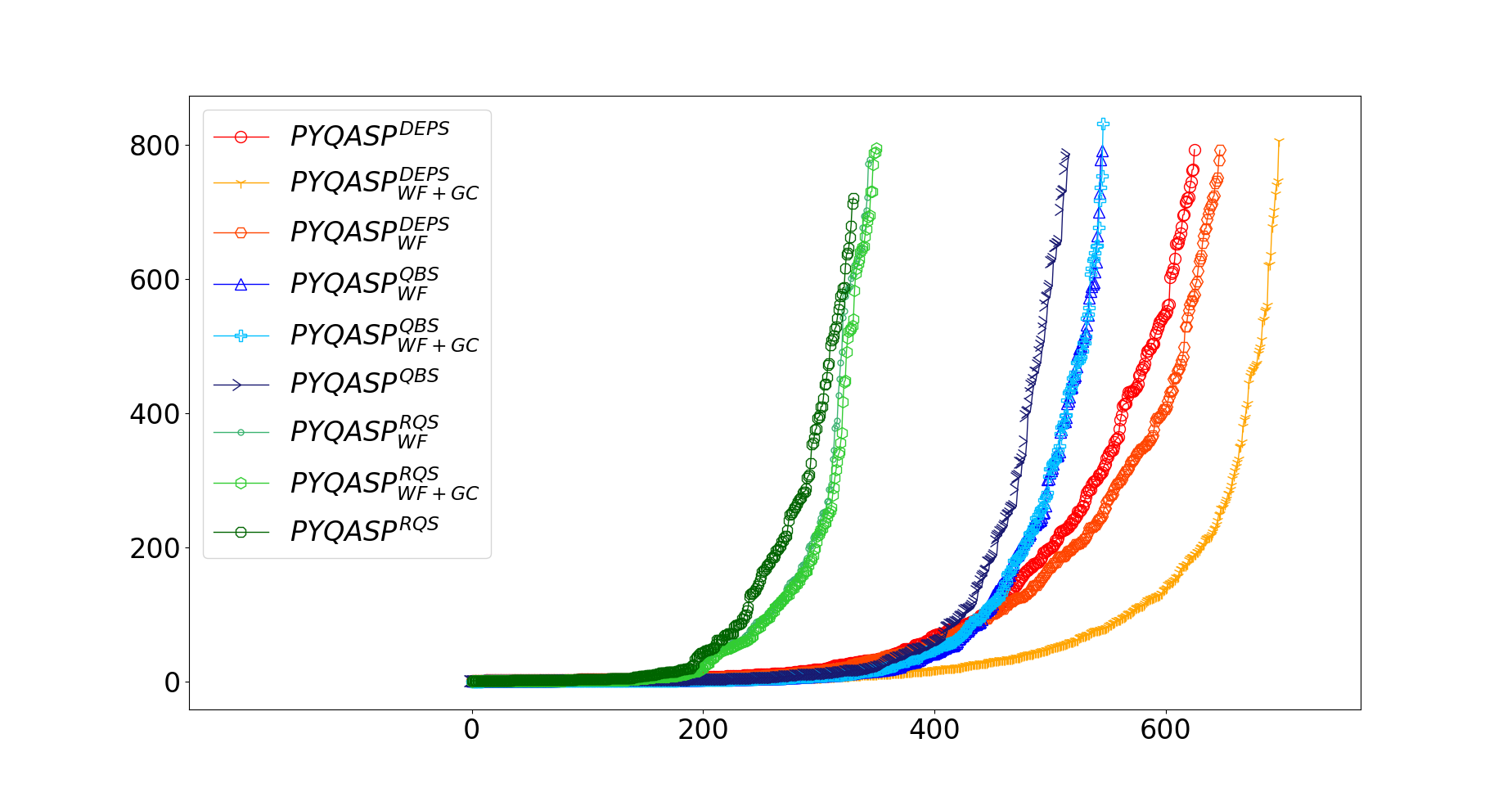}
         \caption{Quantified Boolean Formula (QBF).}
         \label{fig:techniques:qfb}
     \end{subfigure}
        \caption{Analysis of proposed optimizations.}
        \label{fig:techniques}
\end{figure}

\myparagraph{Results.}
Obtained results are summarized in Figure~\ref{fig:techniques}, which aggregates the performance of each compared method in four cactus plots, one per considered problem. Recall that, a line in a cactus plot contains a point $(x,y)$ whenever the corresponding system solves at most $x$ instances in $y$ seconds.

We first observe that the different back-ends are preferable depending on the benchmark domain.
In particular, RQS is the fastest option in AC and PAR (see Figures~\ref{fig:techniques:ac}-\ref{fig:techniques:par}), QBS is the fastest in MMC (see Figure~\ref{fig:techniques:mmc}), and DEPS in QBF (see Figure~\ref{fig:techniques:qfb}).
This behavior confirms the findings of Amendola et al.~\shortcite{DBLP:conf/lpnmr/AmendolaCRT22}. 
%
%

The well-founded optimization allows to solve more instances and in less time in AC, MMC, and QBF benchmarks, independently of the back-end; 
whereas, the identification of \textit{guess\&check} programs pays off in terms of solved instances in QBF and MMC, again independently of the back-end solver. 
The two techniques combine their positive effects in MMC, PAR, and QBF.
In particular,  \pyqaspv{DEPS}{WF+GC} solves 25 more instances than \pyqaspv{DEPS}{} in MC, and 73 in QBF; moreover, \pyqaspv{RQS}{WF+GC} solves 20 more instances than \pyqaspv{RQS}{} in MC, and 20 in QBF. (cfr. Table~\ref{tbl:overall:par2} in \ref{app:exp}.)
%
However, the application of the \textit{guess\&check} optimization has a negative effect on AC, since the well-founded operator, applied to the rewritten program, is no longer able to derive some simplifications that instead can be derived from the original program. 
For this reason, \pyqaspv{RQS}{WF} is the best option in AC, solving 29 instances more than \pyqaspv{RQS}{}.

All in all, the results summarized in Figure~\ref{fig:techniques} confirm the efficacy of both well-founded optimization and identification of \textit{guess\&check} programs.

\subsection{Comparison with \qasp}

\myparagraph{Compared methods.} We compare the best variants of \pyqasp identified in the previous subsection with \qasp running the same back-end QBF solvers. 
As before, the selected back-end is identified by a superscript.
In addition, we run a version of \pyqasp capable of selecting automatically a suitable back-end solver for each instance, denoted by \pyqaspv{AUTO}{}. 
This latter was obtained by applying to \pyqasp the methodology used in the ME-ASP solver~\cite{DBLP:journals/tplp/MarateaPR14} for ASP. 
In particular, we measured some syntactic program features, the ones of ME-ASP augmented with the number of quantifiers, existential (resp. universal) atoms count, and existential (resp. universal) quantifiers to characterize \qasp instances.
Then, we used the \textit{random forest} classification algorithm. We sampled about 25\% of the instances (i.e., 1094 instances uniquely solved) from all benchmark domains, and split in 30\% test set (329 instances) and 70\% training set (765 instances), obtaining: 98\% accuracy, 95\% recall, 97\% of f-measure, which is acceptable. 
As it is customary in the literature, to assess on the field the efficacy of the algorithm selection strategy, we also computed the Virtual Best Solver (VBS). VBS is the \textit{ideal} system one can obtain by always selecting the best solver for each instance.

\myparagraph{Results.}
The results we obtained are reported in the cactus plot of Figure~\ref{fig:comparison:qasp}.
First of all, we note that \pyqasp is faster and solves more instances than \qasp no matter the back-end solver. In particular, \pyqaspv{DEPS}{WF+GC} solves 186 instances more than \qaspv{DEPS}{}, \pyqaspv{RQS}{WF} solves 4 instances more than \qaspv{RQS}{}, and \pyqaspv{QBS}{WF} solves 21 instances more than \qaspv{QBS}{}. 

Diving into the details, we observed that \pyqasp also uses less memory on average than \qasp. Indeed, \qasp used more than 12GB in some instances of PAR and AC, whereas \pyqasp never exceeded the memory limit in these domains. 
This is due to a combination of factors. On the one hand, \pyqasp never caches the entire program in main memory; on the other hand, the formulas built by \pyqasp are smaller than the ones of \qasp and this causes the back-end QBF solver to use less memory and be faster during the search.
(More detailed data on time and memory usage are available in \ref{app:exp}.)

Finally, as one might expect, the best solving method is \pyqaspv{AUTO}{}. 
Comparing \pyqaspv{AUTO}{} with the VBS there is only a small gap (38 instances overall).
In particular, we observe that, in the majority of cases, the selector is able to pick the best method; it sometimes misses a suitable back-end (especially in MMC which is the smallest and less represented domain in the training set). 
As a result, \pyqaspv{AUTO}{} is generally effective in combining the strengths of all the back-end solvers. 
Indeed, \pyqaspv{AUTO}{} solves 363 instances more than \pyqaspv{DEPS}{WF+GC} (i.e., the best variant of \pyqasp with fixed back-end) and 414 instances more than \qaspv{RQS}{} (i.e., the best variant of \qasp).

\begin{figure}
     \centering
     \begin{subfigure}[b]{0.49\textwidth}
         \centering
         \includegraphics[width=\textwidth,trim=60 40 80 80,clip]{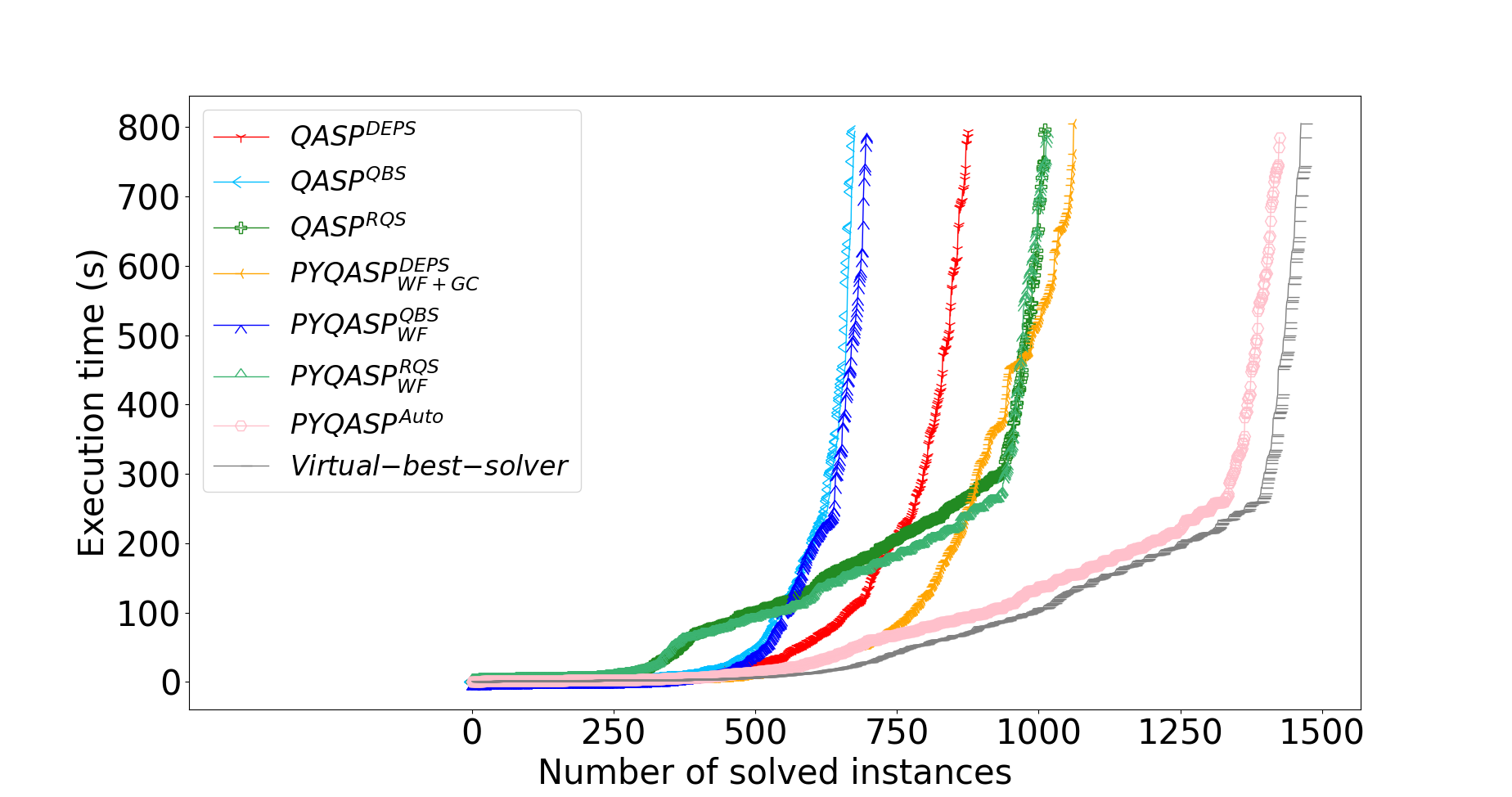}
         \caption{Comparison with \qasp.}
         \label{fig:comparison:qasp}
     \end{subfigure}
     \hfill
     \begin{subfigure}[b]{0.49\textwidth}
         \centering
         \includegraphics[width=\textwidth,trim=60 40 80 70,clip]{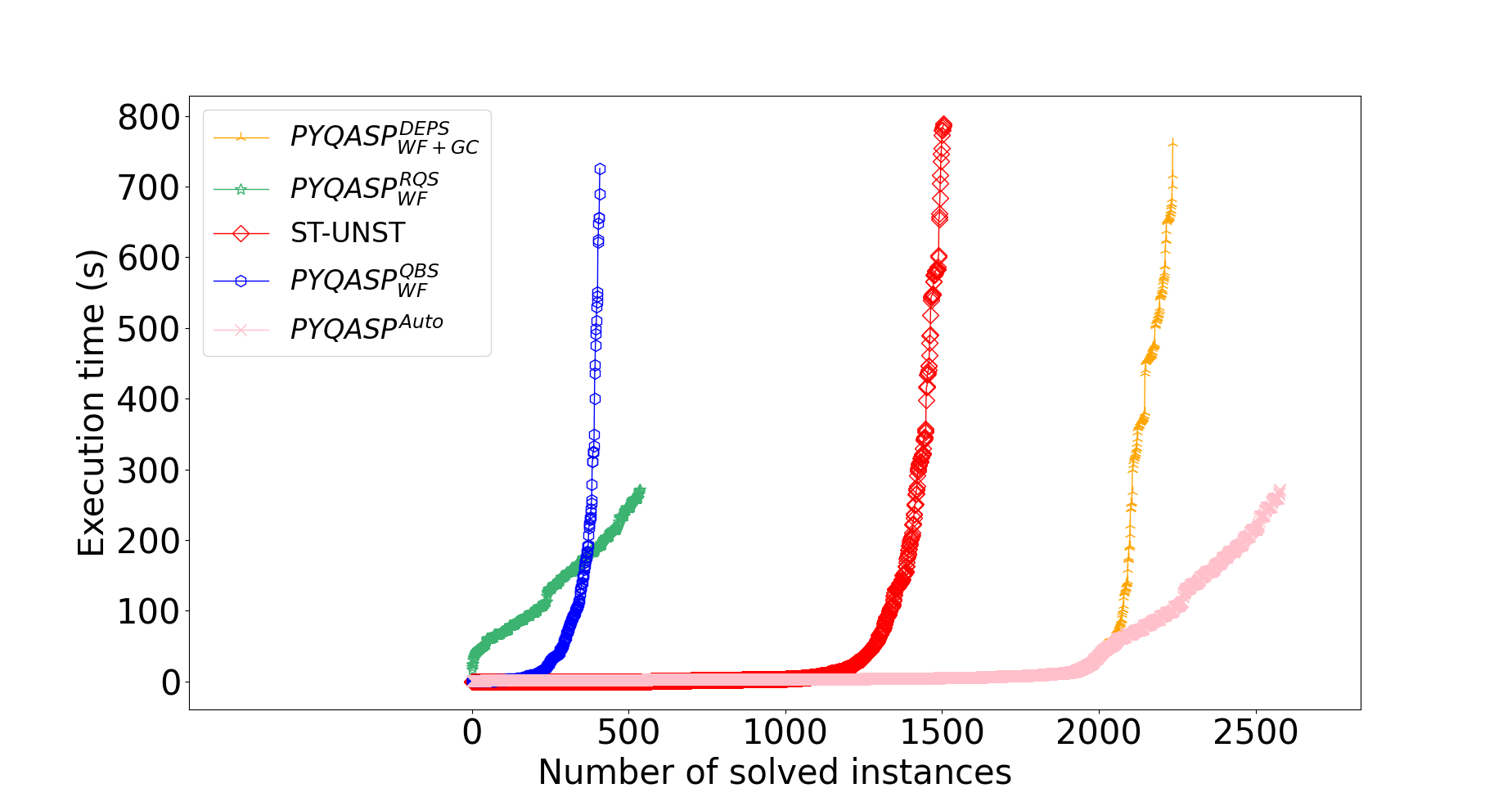}
         \caption{Comparison with \stunst.}
         \label{fig:comparison:stunst}
     \end{subfigure}
        \caption{Comparison with \qasp and \stunst.}
        \label{fig:comparison}
\end{figure}

\subsection{Comparison with Stable-unstable}\label{sec:experiments:stunst}
In this section, we compare \pyqasp with an efficient implementation of the stable-unstable semantics by Janhunen~\shortcite{DBLP:conf/padl/Janhunen22} on common benchmarks.

\myparagraph{Compared methods.} In this comparison, we considered the best fixed back-end variants of \pyqasp, \pyqaspv{AUTO}{}, with Janhunen's solver~\shortcite{DBLP:conf/padl/Janhunen22}, which is labeled \stunst.  

\myparagraph{Benchmarks}. \stunst can solve only problems on the second level of the PH (more on this in Section~\ref{sec:related}).
Thus, to perform a fair comparison, we considered in addition to PAR (the only problem in our suite having suitable complexity), a set of hard 2-QBF instances generated according to the method by Amendola et al.~\shortcite{DBLP:journals/ai/AmendolaRT20}, and the \textit{point of no return} (PONR) benchmark introduced by Janhunen~\shortcite{DBLP:conf/padl/Janhunen22} to assess \stunst.

\myparagraph{Results.} The results are summarized in the cactus plot of Figure~\ref{fig:comparison:stunst}. 
(More details in \ref{app:exp}.)
Analysing the results in each domain, we report that \stunst solves 60 instances of PAR, where the best fixed back-end version of \pyqasp (namely, \pyqaspv{RQS}{WF}) solves 442.
In PONR, \stunst solves 30 instances, where \pyqaspv{RQS}{WF} solves 94.
In QBF, \stunst solves 1416 instances, where \pyqaspv{DEPS}{WF+GC} solves 2048.
Finally, \pyqaspv{AUTO}{} is the best method overall, solving a total of 2578 instances, that is 1072 more instances than \stunst, which solves 1506 overall.
\section{Related Work}\label{sec:related}
The most closely-related work is the one proposed by Amendola et al.~\shortcite{DBLP:conf/lpnmr/AmendolaCRT22} where \qasp, the first implementation of a solver for \ASPQ, was introduced. 
First, we observe that both \pyqasp and \qasp are based on the translation from \ASPQ to QBF introduced by Amendola et al.~\shortcite{DBLP:conf/lpnmr/AmendolaCRT22}), 
both resort to the lp2* tools for converting ASP programs to CNF formulas~\cite{DBLP:conf/ecai/Janhunen04,DBLP:journals/ki/Janhunen18a}, and both can be configured with several back-end QBF solvers.
\qasp is implemented in Java, whereas \pyqasp is implemented in Python, which proved to be a very flexible and handy language to implement the composition of tools that is needed to develop a QBF-based system for \ASPQ. 
\qasp processes the entire \ASPQ program rewriting in main memory, whereas \pyqasp implements a more memory-aware algorithm that keeps at most one subprogram in main memory. 
This implementation choice was empirically demonstrated to overcome the high memory usage limiting the performance of \qasp described by Amendola et al.~\shortcite{DBLP:conf/lpnmr/AmendolaCRT22}.
\qasp uses gringo~\cite{DBLP:conf/lpnmr/GebserKKS11} as grounder, whereas \pyqasp can be configured to use both gringo and iDLV~\cite{DBLP:journals/tplp/CalimeriDFPZ20}, and offers an interface that makes it easier to integrate external QBF solvers.
It is worth pointing out that \pyqasp supports novel rewriting  techniques that result in more efficient encoding in QBF (see Section~\ref{sec:newenc}) that are absent in \qasp.

%
Concerning closely-related formalisms that feature an implementation, we mention the stable-unstable semantics~\cite{DBLP:journals/tplp/BogaertsJT16}, and quantified answer set semantics~\cite{DBLP:journals/tplp/FandinnoLRSS21}.


The stable-unstable semantics was first supported by a proof of concept prototype~\cite{DBLP:journals/tplp/BogaertsJT16}; later, Janhunen~\shortcite{DBLP:conf/padl/Janhunen22} proposed an implementation based on a rewriting to plain ASP. 
The implementation proposed by Janhunen~\shortcite{DBLP:conf/padl/Janhunen22} employed ASPTOOLS for some pre-processing, but the transformations and the solving techniques are different w.r.t. \pyqasp. 
Stable-unstable semantics can be used to model problems in the second level of the PH, thus our implementation can handle problems of higher complexity.
From the usage point of view, we observe that in the system of Janhunen~\shortcite{DBLP:conf/padl/Janhunen22} the user is required to define the interface of modules (by means of ASP programs) and to manually combine the tool chain, whereas in \pyqasp this is done in a more accessible way.
An empirical comparison of \pyqasp with Janhunen's system~\shortcite{DBLP:conf/padl/Janhunen22} is provided in Section~\ref{sec:experiments:stunst}.


The quantified answer set semantics \cite{DBLP:journals/tplp/FandinnoLRSS21} was also implemented by resorting to a translation to QBF~\cite{DBLP:journals/tplp/FandinnoLRSS21}.
However, the difference in the semantics of quantifiers results in a quite  different translation to QBF with respect to \ASPQ. 
Translations from  quantified answer set semantics to \ASPQ and back were proposed by Fandinno et al.~\shortcite{DBLP:journals/tplp/FandinnoLRSS21} but never implemented.

Finally, we refer the reader to the works proposed by Amendola et.al~\shortcite{DBLP:journals/tplp/AmendolaRT19}, and Fandinno et al.~\shortcite{DBLP:journals/tplp/FandinnoLRSS21} for an exhaustive comparison of the \ASPQ language with alternative formalisms and semantics.


\section{Conclusion}\label{sec:conclusion}
An important aspect that can boost the adoption of \ASPQ as a practical tool for developing applications is the availability of more efficient implementations. 
In this paper, we present \pyqasp, a new system for \ASPQ that features both a memory-aware implementation in Python and a new optimized translation of \ASPQ programs in QBF. 
In particular, \pyqasp exploits the well-founded operator to simplify \ASPQ programs and can recognize a (popular) class of \ASPQ programs that can be encoded directly in CNF, and thus do not require to perform any additional normalization to be handled by QBF solvers.
Moreover, \pyqasp is able to select automatically a suitable back-end for the given input, and can deliver steady performance over varying problem instances.
\pyqasp outperforms \qasp, the first implementation of \ASPQ, and pushes forward the state of the art in \ASPQ solving. 

As future work, we plan to further optimize \pyqasp by providing more efficient encodings in QBFs, and improve the algorithm selection model with extended training and a deeper tuning of parameters.


\bibliographystyle{acmtrans}
\bibliography{refs}

\begin{thebibliography}{}

\bibitem[\protect\citeauthoryear{Amendola, Cuteri, Ricca, and
  Truszczynski}{Amendola et~al\mbox{.}}{2022}]{DBLP:conf/lpnmr/AmendolaCRT22}
{\sc Amendola, G.}, {\sc Cuteri, B.}, {\sc Ricca, F.}, {\sc and} {\sc
  Truszczynski, M.} 2022.
\newblock Solving problems in the {PH} with {ASP(Q)}.
\newblock In {\em Proceedings of {LPNMR}}. LNCS, vol. 13416. Springer,
  373--386.

\bibitem[\protect\citeauthoryear{Amendola, Dodaro, Faber, and Ricca}{Amendola
  et~al\mbox{.}}{2021}]{DBLP:journals/ai/AmendolaDFR21}
{\sc Amendola, G.}, {\sc Dodaro, C.}, {\sc Faber, W.}, {\sc and} {\sc Ricca,
  F.} 2021.
\newblock Paracoherent answer set computation.
\newblock {\em Artif. Intell.\/}~{\em 299}, 103519.

\bibitem[\protect\citeauthoryear{Amendola, Ricca, and Truszczynski}{Amendola
  et~al\mbox{.}}{2019}]{DBLP:journals/tplp/AmendolaRT19}
{\sc Amendola, G.}, {\sc Ricca, F.}, {\sc and} {\sc Truszczynski, M.} 2019.
\newblock Beyond {NP:} quantifying over answer sets.
\newblock {\em TPLP\/}~{\em 19,\/}~5-6, 705--721.

\bibitem[\protect\citeauthoryear{Amendola, Ricca, and Truszczynski}{Amendola
  et~al\mbox{.}}{2020}]{DBLP:journals/ai/AmendolaRT20}
{\sc Amendola, G.}, {\sc Ricca, F.}, {\sc and} {\sc Truszczynski, M.} 2020.
\newblock New models for generating hard random boolean formulas and
  disjunctive logic programs.
\newblock {\em Artif. Intell.\/}~{\em 279}.

\bibitem[\protect\citeauthoryear{Bogaerts, Janhunen, and Tasharrofi}{Bogaerts
  et~al\mbox{.}}{2016}]{DBLP:journals/tplp/BogaertsJT16}
{\sc Bogaerts, B.}, {\sc Janhunen, T.}, {\sc and} {\sc Tasharrofi, S.} 2016.
\newblock Stable-unstable semantics: Beyond {NP} with normal logic programs.
\newblock {\em TPLP\/}~{\em 16,\/}~5-6, 570--586.

\bibitem[\protect\citeauthoryear{Brewka, Eiter, and Truszczynski}{Brewka
  et~al\mbox{.}}{2011}]{DBLP:journals/cacm/BrewkaET11}
{\sc Brewka, G.}, {\sc Eiter, T.}, {\sc and} {\sc Truszczynski, M.} 2011.
\newblock Answer set programming at a glance.
\newblock {\em Commun. {ACM}\/}~{\em 54,\/}~12, 92--103.

\bibitem[\protect\citeauthoryear{Calimeri, Dodaro, Fusc{\`{a}}, Perri, and
  Zangari}{Calimeri et~al\mbox{.}}{2020}]{DBLP:journals/tplp/CalimeriDFPZ20}
{\sc Calimeri, F.}, {\sc Dodaro, C.}, {\sc Fusc{\`{a}}, D.}, {\sc Perri, S.},
  {\sc and} {\sc Zangari, J.} 2020.
\newblock Efficiently coupling the {I-DLV} grounder with {ASP} solvers.
\newblock {\em TPLP\/}~{\em 20,\/}~2, 205--224.

\bibitem[\protect\citeauthoryear{Dantsin, Eiter, Gottlob, and Voronkov}{Dantsin
  et~al\mbox{.}}{2001}]{DBLP:journals/csur/DantsinEGV01}
{\sc Dantsin, E.}, {\sc Eiter, T.}, {\sc Gottlob, G.}, {\sc and} {\sc Voronkov,
  A.} 2001.
\newblock Complexity and expressive power of logic programming.
\newblock {\em {ACM} Comput. Surv.\/}~{\em 33,\/}~3, 374--425.

\bibitem[\protect\citeauthoryear{Eiter and Gottlob}{Eiter and
  Gottlob}{1995}]{DBLP:journals/amai/EiterG95}
{\sc Eiter, T.} {\sc and} {\sc Gottlob, G.} 1995.
\newblock On the computational cost of disjunctive logic programming:
  Propositional case.
\newblock {\em Ann. Math. Artif. Intell.\/}~{\em 15,\/}~3-4, 289--323.

\bibitem[\protect\citeauthoryear{Erdem, Gelfond, and Leone}{Erdem
  et~al\mbox{.}}{2016}]{DBLP:journals/aim/ErdemGL16}
{\sc Erdem, E.}, {\sc Gelfond, M.}, {\sc and} {\sc Leone, N.} 2016.
\newblock Applications of answer set programming.
\newblock {\em {AI} Magazine\/}~{\em 37,\/}~3, 53--68.

\bibitem[\protect\citeauthoryear{Faber and Morak}{Faber and
  Morak}{2022}]{DBLP:conf/lpnmr/0001M22}
{\sc Faber, W.} {\sc and} {\sc Morak, M.} 2022.
\newblock Evaluating epistemic logic programs via answer set programming with
  quantifiers.
\newblock In {\em HYDRA/RCRA@LPNMR}. {CEUR} WS, vol. 3281. 78--89.

\bibitem[\protect\citeauthoryear{Faber, Morak, and Chrpa}{Faber
  et~al\mbox{.}}{2022}]{DBLP:conf/padl/FaberMC22}
{\sc Faber, W.}, {\sc Morak, M.}, {\sc and} {\sc Chrpa, L.} 2022.
\newblock Determining action reversibility in {STRIPS} using asp with
  quantifiers.
\newblock In {\em {PADL}}. LNCS, vol. 13165. Springer, 42--56.

\bibitem[\protect\citeauthoryear{Fandinno, Laferri{\`{e}}re, Romero, Schaub,
  and Son}{Fandinno et~al\mbox{.}}{2021}]{DBLP:journals/tplp/FandinnoLRSS21}
{\sc Fandinno, J.}, {\sc Laferri{\`{e}}re, F.}, {\sc Romero, J.}, {\sc Schaub,
  T.}, {\sc and} {\sc Son, T.~C.} 2021.
\newblock Planning with incomplete information in quantified answer set
  programming.
\newblock {\em TPLP\/}~{\em 21,\/}~5, 663--679.

\bibitem[\protect\citeauthoryear{Gebser, Kaminski, K{\"{o}}nig, and
  Schaub}{Gebser et~al\mbox{.}}{2011}]{DBLP:conf/lpnmr/GebserKKS11}
{\sc Gebser, M.}, {\sc Kaminski, R.}, {\sc K{\"{o}}nig, A.}, {\sc and} {\sc
  Schaub, T.} 2011.
\newblock Advances in \emph{gringo} series 3.
\newblock In {\em {LPNMR} 2011. Proceedings}. LNCS, vol. 6645. Springer,
  345--351.

\bibitem[\protect\citeauthoryear{Gebser, Leone, Maratea, Perri, Ricca, and
  Schaub}{Gebser et~al\mbox{.}}{2018}]{DBLP:conf/ijcai/GebserLMPRS18}
{\sc Gebser, M.}, {\sc Leone, N.}, {\sc Maratea, M.}, {\sc Perri, S.}, {\sc
  Ricca, F.}, {\sc and} {\sc Schaub, T.} 2018.
\newblock Evaluation techniques and systems for answer set programming: a
  survey.
\newblock In {\em Proceedings of {IJCAI} 2018}. ijcai.org, 5450--5456.

\bibitem[\protect\citeauthoryear{Gebser, Maratea, and Ricca}{Gebser
  et~al\mbox{.}}{2017}]{DBLP:journals/jair/GebserMR17}
{\sc Gebser, M.}, {\sc Maratea, M.}, {\sc and} {\sc Ricca, F.} 2017.
\newblock The sixth answer set programming competition.
\newblock {\em J. Artif. Intell. Res.\/}~{\em 60}, 41--95.

\bibitem[\protect\citeauthoryear{Gelfond and Lifschitz}{Gelfond and
  Lifschitz}{1991}]{DBLP:journals/ngc/GelfondL91}
{\sc Gelfond, M.} {\sc and} {\sc Lifschitz, V.} 1991.
\newblock Classical negation in logic programs and disjunctive databases.
\newblock {\em New Gener. Comput.\/}~{\em 9,\/}~3/4, 365--386.

\bibitem[\protect\citeauthoryear{Janhunen}{Janhunen}{2004}]{DBLP:conf/ecai/Janhunen04}
{\sc Janhunen, T.} 2004.
\newblock Representing normal programs with clauses.
\newblock In {\em Proceedings of ECAI'2004.}, {R.~L. de~M{\'{a}}ntaras} {and}
  {L.~Saitta}, Eds. {IOS} Press, 358--362.

\bibitem[\protect\citeauthoryear{Janhunen}{Janhunen}{2018}]{DBLP:journals/ki/Janhunen18a}
{\sc Janhunen, T.} 2018.
\newblock Cross-translating answer set programs using the {ASPTOOLS}
  collection.
\newblock {\em K{\"{u}}nstliche Intell.\/}~{\em 32,\/}~2-3, 183--184.

\bibitem[\protect\citeauthoryear{Janhunen}{Janhunen}{2022}]{DBLP:conf/padl/Janhunen22}
{\sc Janhunen, T.} 2022.
\newblock Implementing stable-unstable semantics with {ASPTOOLS} and clingo.
\newblock In {\em {PADL} 2022, Proceedings}. LNCS, vol. 13165. Springer,
  135--153.

\bibitem[\protect\citeauthoryear{Lifschitz}{Lifschitz}{2002}]{Lifschitz02}
{\sc Lifschitz, V.} 2002.
\newblock Answer set programming and plan generation.
\newblock {\em Artif. Intell.\/}~{\em 138,\/}~1-2, 39--54.

\bibitem[\protect\citeauthoryear{Maratea, Pulina, and Ricca}{Maratea
  et~al\mbox{.}}{2014}]{DBLP:journals/tplp/MarateaPR14}
{\sc Maratea, M.}, {\sc Pulina, L.}, {\sc and} {\sc Ricca, F.} 2014.
\newblock A multi-engine approach to answer-set programming.
\newblock {\em TPLP\/}~{\em 14,\/}~6, 841--868.

\bibitem[\protect\citeauthoryear{Pulina and Seidl}{Pulina and
  Seidl}{2019}]{DBLP:journals/ai/PulinaS19}
{\sc Pulina, L.} {\sc and} {\sc Seidl, M.} 2019.
\newblock The 2016 and 2017 {QBF} solvers evaluations (qbfeval'16 and
  qbfeval'17).
\newblock {\em Artif. Intell.\/}~{\em 274}, 224--248.

\bibitem[\protect\citeauthoryear{{Van Gelder}, Ross, and Schlipf}{{Van Gelder}
  et~al\mbox{.}}{1991}]{DBLP:journals/jacm/GelderRS91}
{\sc {Van Gelder}, A.}, {\sc Ross, K.~A.}, {\sc and} {\sc Schlipf, J.~S.} 1991.
\newblock The well-founded semantics for general logic programs.
\newblock {\em J. {ACM}\/}~{\em 38,\/}~3, 620--650.

\end{thebibliography}

\newpage

\appendix

\section{Additional experimental data}\label{app:exp}
This section reports some more data on the Experiments described in Section~\ref{sec:experiments}, reported here to provide a more detailed view on our results for the reviewers.

Table~\ref{tbl:overall:par2} shows the PAR2 score for all the compared systems. Recall that the PAR-2 score of a solver is defined as the sum of all execution times for solved instances and  2 times the timeout for unsolved ones. The lower the score, the better the performance. 

Table~\ref{tbl:overall:memory} shows the average memory usage for all the compared systems. Memory usage (measured in MB) is aggregated for instances solved within the time limit (Complete), instances that exceeded the time limit (Timeout) and over all the instances (Total). The lower the memory usage, the better the performance. 

Table~\ref{tbl:overall:bench} shows, for each system variant, the number of solved instances, timeouts and memory out for each benchmark and also the total number of solved instances overall.

Table~\ref{tab:comparison} reports the comparison with \qasp and \stunst implementation (respectively \ref{tab:comparison:qasp} \ref{tab:comparison:stunst}). We considered the best variants of \pyqaspv{}{} against the other systems, \qaspv{}{} with supported back-end solvers and \stunst. For each of them, the number of solved instances for each benchmark and the overall number of solved instances is reported.
\begin{table}[h!]
 \centering
 {\begin{tabular}{@{\extracolsep{\fill}}lr r r r r r r}
 \hline\hline
 Solver & Par && Arg.Cohe. && Minmax Cli. && QBF\\
    \hline
    
    \qaspv{DEPS}{}        & 771,235.84           && 321,279.33            && 36,427.46            && 546,292.46             \\   
    \pyqaspv{DEPS}{}      & 645,221.78           && 348,019.17            && 43,217.22            && 654,649.81             \\   
    \pyqaspv{DEPS}{WF}    & 589,804.20           && 299,694.65            && 21,414.65            && 619,824.33             \\   
    \pyqaspv{DEPS}{WF+GC} & 588,632.92           && 314,896.34            && 1,322.85             && \textbf{513,758.82}    \\   
    \hline
    \qaspv{QBS}{}         & 822,400.00           && 369,190.24            && 553.76              && 768,291.40             \\   
    \pyqaspv{QBS}{}       & 822,400.00           && 382,814.62            && 587.20              && 795,651.07             \\   
    \pyqaspv{QBS}{WF}     & 822,400.00           && 362,715.16            && \textbf{458.76}     && 749,836.74             \\   
    \pyqaspv{QBS}{WF+GC}  & 822,400.00           && 481,788.14            && 524.75              && 749,905.75              \\  
    \hline
    \qaspv{RQS}{}         & 190,769.32           && 231,845.74            && 38,109.56            && 1,054,759.98            \\   
    \pyqaspv{RQS}{}       & 239,202.25           && 261,084.28            && 41,908.61            && 1,091,812.72            \\   
    \pyqaspv{RQS}{WF}     & 180,695.99           && \textbf{218,876.82}   && 36,427.19            && 1,067,281.23            \\   
    \pyqaspv{RQS}{WF+GC}  & \textbf{180,692.43}  && 499,799.74            && 8,876.10             && 1,059,731.92            \\
       
    \hline\hline
    \end{tabular}}
\caption{PAR-2 score in seconds for each system variants on: (i) Paracoherent ASP (Par. Comp, Par. Rand.); (ii) Argumentation Coherence (Arg. Cohe.); (iii) Minmax Clique (Minmax Cli.); (iv) Quantified Boolean Formula (QBF) }
\label{tbl:overall:par2}
\end{table}
\begin{table}[t!]
 \centering
 {\begin{tabular}{@{\extracolsep{\fill}}lr r r}
 \hline\hline
    Solver & Complete & Timeout & Total\\
    \hline
    \qaspv{DEPS}{}        & 359.06 & 1281.35 & 856.37 \\
    \pyqaspv{DEPS}{}      & 108.41 & 657.65  & 394.71 \\
    \pyqaspv{DEPS}{WF}    & 168.15 & 607.33  & \textbf{374.97} \\
    \pyqaspv{DEPS}{WF+GC} & 170.56 & 789.57  & 445.17 \\
    \hline
    \qaspv{QBS}{}         & 553.15 & 1892.40 & 1701.98 \\
    \pyqaspv{QBS}{}       & 272.02 & 1022.99 & \textbf{874.03}  \\
    \pyqaspv{QBS}{WF}     & 343.78 & 1232.27 & 1014.53 \\
    \pyqaspv{QBS}{WF+GC}  & 379.94 & 1304.55 & 1077.01 \\
    \hline
    \qaspv{RQS}{}         & 356.44 & 1409.87 & 847.78 \\
    \pyqaspv{RQS}{}       & 103.73 & 701.44  & 398.42 \\
    \pyqaspv{RQS}{WF}     & 152.70 & 634.67  & \textbf{380.00} \\
    \pyqaspv{RQS}{WF+GC}  & 177.27 & 632.06  & 432.74  \\
    \hline\hline
    \end{tabular}}
\caption{Average memory consumption in megabyte for each system variant}
\label{tbl:overall:memory}
\end{table}
\begin{table}[b]
\caption{Comparison of system variants on Paracoherent ASP (PAR), Argumentation Coherence (AC), Minmax Clique (MMC), QBF and overall (TOTAL). }
\label{tbl:overall:bench}
 \centering \vspace{0.2cm}
 { \scriptsize \begin{tabular}{@{\extracolsep{\fill}}lccc c ccc c ccc c ccc c c}
 \multicolumn{1}{l}{\multirow{2}{*}{Solver}} & \multicolumn{3}{c}{PAR}&& \multicolumn{3}{c}{AC} && \multicolumn{3}{c}{MMC} && \multicolumn{3}{c}{QBF} && \multicolumn{1}{c}{TOTAL}  \\ 
    \cline{2-4} \cline{6-8} \cline{10-12} \cline{14-16} \cline{17-18}
    
     & \multicolumn{1}{c}{\#SO} & \multicolumn{1}{c}{\#MO} & \multicolumn{1}{c}{\#TO} &&\multicolumn{1}{c}{\#SO} & \multicolumn{1}{c}{\#MO} & \multicolumn{1}{c}{\#TO}&&\multicolumn{1}{c}{\#SO} & \multicolumn{1}{c}{\#MO} & \multicolumn{1}{c}{\#TO}&&\multicolumn{1}{c}{\#SO} & \multicolumn{1}{c}{\#MO} & \multicolumn{1}{c}{\#TO} && \multicolumn{1}{c}{\#SO}\\
    \hline
    \qaspv{DEPS}{}         & 37           &  0  &  477 && 133          & 0  & 193  && 25          & 0 & 20 && 682          & 1  & 309 && 877  \\
    \pyqaspv{DEPS}{}       & 157          &  0  &  357 && 117          & 0  & 209  && 20          & 0 & 25 && 626          & 1  & 365 && 920  \\
    \pyqaspv{DEPS}{WF}     & 189          &  0  &  325 && 148          & 0  & 178  && 35          & 0 & 10 && 648          & 1  & 343 && 869  \\
    \pyqaspv{DEPS}{WF+GC}  & 189          &  0  &  325 && 130          & 0  & 196  && \textbf{45} & 0 & 0  && 699          & 1  & 292 && 918  \\
    \hline                                                                                                                                   
    \qaspv{QBS}{}          & 0            &  6  &  508 && 102          & 26 & 198  && \textbf{45} & 0 & 0  && 530          & 20 & 442 && 677  \\
    \pyqaspv{QBS}{}        & 0            &  0  &  514 && 96           & 0  & 230  && \textbf{45} & 0 & 0  && 518          & 19 & 455 && 659  \\
    \pyqaspv{QBS}{WF}      & 0            &  0  &  514 && 107          & 0  & 219  && \textbf{45} & 0 & 0  && 546          & 19 & 427 && 698  \\
    \pyqaspv{QBS}{WF+GC}   & 0            &  0  &  514 && 26           & 0  & 300  && \textbf{45} & 0 & 0  && 547          & 13 & 432 && 618   \\ 
    \hline                                                                                                                                   
    \qaspv{RQS}{}          & \textbf{442} &  0  &  72  && 197          & 0  & 129  && 23          & 0 & 22 && 350          & 1  & 641 && 1012 \\
    \pyqaspv{RQS}{}        & \textbf{442} &  0  &  72  && 176          & 0  & 150  && 22          & 0 & 23 && 331          & 1  & 660 && 971  \\
    \pyqaspv{RQS}{WF}      & \textbf{442} &  0  &  72  && 205          & 0  & 121  && 24          & 0 & 21 && 345          & 1  & 646 && 1016 \\
    \pyqaspv{RQS}{WF+GC}   & \textbf{442} &  0  &  72  && 14           & 0  & 312  && 42          & 0 & 3  && 351          & 1  & 640 && 849   \\
    \hline                                                                                                                                 
    \pyqaspv{Auto}{}       & \textbf{442} &  0  &  72  && \textbf{222} & 0  & 104  && 44          & 0 & 1  && \textbf{718} & 7 & 267  && \textbf{\underline{1426}}\\
    \end{tabular}}
\end{table}

\begin{table}[b] \caption{Comparison with \qasp and \stunst: Solved instances.} \label{tab:comparison}
 \centering \vspace{0.2cm}
 {
     \footnotesize 
    \begin{subtable}[h]{0.47\textwidth}
    \begin{tabular}{@{\extracolsep{\fill}}l c c c c c c}

     	Solver & PAR & AC & MMC & QBF & TOTAL  \\
        \hline
        \qaspv{DEPS}{}         & 37  		   & 133  		   & 25 		  & 682 		  & 877  \\
        \pyqaspv{DEPS}{WF+GC}  & 189 		   & 130  		   & \textbf{45}  & 699 		  & 1063 \\
    	\hline
        \qaspv{QBS}{}          & 0   		   & 102 		   & \textbf{45}  & 530 		  & 677  \\
        \pyqaspv{QBS}{WF}      & 0   		   & 107 		   & \textbf{45}  & 546 		  & 698  \\
        \hline
        \qaspv{RQS}{}          & \textbf{442} & 197 		   & 23 		  & 350 		  & 1012 \\
        \pyqaspv{RQS}{WF}      & \textbf{442} & 205 		   & 24 		  & 345 		  & 1016 \\
        \hline
        \pyqaspv{Auto}{}       & \textbf{442} & \textbf{222}   & 44 		  & \textbf{718}  & \textbf{\underline{1426}} \\
        \hline
        \end{tabular}
    \caption{Comparison with \qasp.}
    \label{tab:comparison:qasp}
    
    \end{subtable}
    \hfill
    \begin{subtable}[h]{0.47\textwidth}
        \begin{tabular}{@{\extracolsep{\fill}}l c c c c c}
     	Solver & PAR & PONR & 2-QBF & TOTAL  \\
        \hline
        \stunst       		   & 60           & 30          & 1416          & 1506 \\
        \hline
        \pyqaspv{DEPS}{WF+GC}  & 189          & 0           & \textbf{2048} & 2237 \\
        \hline
        \pyqaspv{QBS}{WF}      & 0            & 63          & 346           & 409  \\
        \hline
        \pyqaspv{RQS}{WF}      & \textbf{442}          & \textbf{94} & 0             & 536   \\
        \hline
        \pyqaspv{Auto}{}       & \textbf{442} & 88          & \textbf{2048} & \textbf{\underline{2578}} \\
        \hline
        \end{tabular}
        \caption{Comparison with \stunst.}
        \label{tab:comparison:stunst}
    \end{subtable}
}
\end{table}

\section{Implementation Details}\label{app:implementation}
This section reports a more detailed description of the process used by \pyqasp to evaluate an \ASPQ program.

\subsection{Base solver}

\textsc{PyQASP} has been entirely developed in Python and it is made by different modules that we will describe in this section.  
The evaluation of an \ASPQ program is, basically, done in two steps that are encoding and solving. In the encoding phase an \ASPQ program is parsed, identifying ASP programs enclosed under the quantifiers' scope. Then, each ASP subprogram $P_i$ passes through the following pipeline:
\begin{enumerate}
    \item \textbf{Rewriting Module}. This modules is designed to compute syntactical properties of $P_i$, in order to check whether it is a Guess\&Check or trivial subprogram and subsequently apply the appropriate rewriting techniques described in this paper. 
    First of all, $P_i$ is rewritten, taking into account a previous Guess\&Check subprogram $P_j$ with $j<i$, if any exists. Then, if the resulting program is trivial, this module returns atoms defined at the current level with an empty program. Otherwise, if it is a Guess\&Check program and it is universally quantified then $P_i$ is split into $G_{P_i}$ and $C_{P_i}$. Moreover, the module computes the result of the transformation $\tau$, introducing a fresh propositional atom $u_i$, that will be used for rewriting the following levels. As a result, it returns the atoms defined in the guess split of the current subprogram with an empty program. In all the other cases this module returns the current program together with symbols defined at the current level.
    \item \textbf{Well-founded Module}. This module computes the well-founded model together with the residual program by means of \texttt{DLV2} as a back-end system and stores the truth values of literals belonging to the well-founded model.
    \item \textbf{CNF Encoder Module}. This module takes as input the residual program produced by the well-founded module and encodes it into a CNF formula. In particular, if the residual program is incoherent then it is encoded as the empty clause that is equivalent to $\bot$ and then it breaks the pipeline. Otherwise, if the residual program is empty it is encoded as an empty CNF. In all the other cases the residual program is encoded into a CNF by means of ASPTOOLS. 
    \item \textbf{QBF Builder}. This module produces the final QBF formula by associating  the symbols produced by \textbf{Rewriting Module} with the respective quantifiers and joining the CNFs produced by the previous module in the final conjunction.
\end{enumerate}
The last step of the encoding phase is combining previous CNFs into the formula $\phi_c$. As a result, a QBF formula in \texttt{QCIR} format is obtained.
The solving step is mainly performed by the solver module, which is a wrapper module for the various QBF solvers. In order to use a solver, the first step in the wrapper is to convert the QCIR formula into an equivalent formula in the solver's input format. Then, the external QBF solver is executed on the converted formula and the final outcome is computed. In the current implementation we provide the following solver wrappers:
\begin{itemize}
    \item \textbf{\texttt{QuabsWapper}}. It uses the QBF solver \texttt{quabs} and doesn't require any format conversion since  the solver directly accepts QCIR formulas.
    \item \textbf{\texttt{RareqsWrapper}} This wrapper uses the QBF solver \texttt{rareqs} whose input format is \texttt{gq}. The conversion from \texttt{QCIR} to \textit{gp} is implemented by the external module \texttt{qcir-conv} provided by (ref to qcir-conv).
    \item \textbf{\texttt{DepqbfWrapper}} It uses the QBF solver \texttt{depqbf} equipped with the QBF pre-processor \texttt{bloqqer}. This solver takes as input formulas in \textit{QDIMACS} format and so, translation to CNF is required. In particular, if all universally quantified subprograms were Guess\&Check then we know that the produced formula is, indeed, in CNF. So a direct mapping into QDIMACS format exists, just reporting quantifiers and clauses of intermediate CNFs. Otherwise, the external module \texttt{qcir-conv} combined with \texttt{fmla} is used in order to translate the input formula into an equivalent QDIMACS one. Note that this translation could introduce extra symbols and clauses leading to a bigger formula.
\end{itemize} 

\subsection{Automatic selection of the back-end}
The automatic back-end selection has been realized by exploiting machine learning models that have been trained on dataset reporting syntactical properties of benchmarks proposed for \ASPQ. 
For this task we extended our system by adding a module (\textsc{aspstats}) that analyzes ground programs during encoding phase and then, a Random Forest Classifier is used to predict the back-end solver to be used. 
In order to train the employed model we considered a dataset containing instances from all our benchmarks: Argumentation Coherence, Paracoherent ASP, Minmax Clique, Point of No Return, QBF and 2QBF. In particular, for each instance the features reported in Table \ref{tab:training_features} have been computed by using \textsc{aspstats} module. 
As required by the ME-ASP methodology, training set has been constructed by considering only those instances that have been solved exactly by one back-end solver that indeed is the target label, and considered the best oracles available as labels for multinomial classification.
Regarding training phase we used a Random Forest Classifier made of 100 trees that have been trained by using Gini impurity criterion and bootstrap sampling technique.   
\begin{table}[t]
    \centering
    \hfill
    \begin{tabular}{l|l}
        $R$         &Rule count                                        \\
        $A$         &Number of atoms                                   \\
        $(R/A)$     &Ratio between rules count and atoms count         \\
        $(R/A)^2$   &Squared ratio between rules count and atoms count \\
        $(R/A)^3$   &Cube ratio between rules count and atoms count    \\
        $(A/R)$     &Ratio between atoms count and rules count         \\
        $(A/R)^2$   &Squared ratio between atoms count and rules count \\
        $(A/R)^3$   &Cube ratio between atoms count and rules count    \\
        $R1$        &Rule with body of length 1                        \\
        $R2$        &Rule with body of length 2                        \\
        $R3$        &Rule with body of length 3                        \\
        $PR$        &Positive rule count                               \\
        $F$         &Normal facts count                                \\
        $DF$        &Disjunctive facts count                           \\
        $NR$        &Normal rule count                                 \\
        $NC$        &Constraint count                                  \\
        $VF$        &Universal atoms count                             \\
        $VE$        &Existantial atoms count                           \\
        $QF$        &Universial levels count                           \\
        $QE$        &Existantial levels count                          \\
        $QL$        &Quantification levels count                       \\
    \end{tabular}
    \caption{\textsc{aspstats} features}
    \label{tab:training_features}
\end{table}
\section{Examples encoding of \qasp program into qbf formula}
Consider a \ASPQ program $\Pi$ of the form: $\exists P_1 \forall P_2 : C$, where
$$
\begin{array}{lllll}
     
     P_1                    && P_2                            && C                  \\
     \{a;b\}\leftarrow      && c\leftarrow not \ a, \ not \ b && \leftarrow e, \ c  \\
     \leftarrow a,\ not \ b && d\leftarrow a, \ b             && \leftarrow e, \ d  \\
                            && \{e\}\leftarrow                &&                    
\end{array}
$$
The first step of the encoding produces the following programs by adding interface from previous levels:
$$
\begin{array}{lllll}
     G_1                    && G_2                            && G_3                      \\
     \{a;b\}\leftarrow      && c\leftarrow not \ a, \ not \ b && \leftarrow e, \ c  \\
     \leftarrow a,\ not \ b && d\leftarrow a, \ b             && \leftarrow e, \ d  \\
                            && \{e\}\leftarrow                && \{a;b;c;d;e\}\leftarrow  \\
                            && \{a;b\}\leftarrow              &&                    
\end{array}
$$
The resulting CNF encodings are the following:
$$
\begin{array}{lll}
     CNF(G_1): && (b \vee aux_1) \wedge (-b \vee -aux1) \wedge (a \vee aux2) \wedge (-a \vee -aux2) \wedge (-a \vee b) \\
     CNF(G_2): && (a \vee aux_3) \wedge (-a \vee -aux_3) \wedge (b \vee aux_4) \wedge (-b \vee -aux_4) \wedge \\
               && (d \vee -b \vee -a) \wedge (-d \vee b) \wedge (-d \vee a) \wedge \\
               && (c \vee a \vee b) \wedge (-c \vee -b) \wedge (-c \vee -a) \wedge \\ 
               && (e \vee aux_5) \wedge (-e \vee -aux_5)\\
     CNF(G_3): && (a \vee aux_6) \wedge (-a \vee -aux_6) \wedge (b \vee aux_7) \wedge (-b \vee -aux_7) \wedge \\
               && (c \vee aux_8) \wedge (-c \vee -aux_8) \wedge (d \vee aux_9) \wedge (-d \vee -aux_9) \wedge \\
               && (e \vee aux_{10}) \wedge (-e \vee -aux_{10}) \wedge                                               \\
               && (-e \vee -d) \wedge (-e \vee -c)                                                            \\
\end{array}
$$
where $aux_i$ are hidden atoms are fresh propositional variables introduced by translation
The final qbf formula $\Phi(\Pi)$:
$$
\begin{array}{l}
     \exists \ a,\ b,\ aux_1,\ aux_2   \\
     \forall \ c,\ d,\ e, \ aux_3,\ aux_4, \ aux_5 \\
     \exists \ aux_6,\ aux_7, \ aux_8, \ aux_9,\ aux_{10} \\
     ((\phi_1 \leftrightarrow CNF(G_1)) \wedge (\phi_2 \leftrightarrow CNF(G_2)) \wedge (\phi_3 \leftrightarrow CNF(G_3))) \wedge \\
     (\phi_1 \wedge (\phi_2 \vee \phi_3))
\end{array}
$$
\section{Example of Guess\&Check rewriting procedure}
Consider a \ASPQ program $\Pi$ of the form: $\forall P_1 \exists P_2 : C$, where $C$ is empty and
$$
\begin{array}{lll}
     
     P_1                     && P_2                            \\
     \{a(1);a(2)\}\leftarrow && b(1) \leftarrow                \\
     \leftarrow a(1), \ a(2) && b(2) \leftarrow                \\
                             && c(1) \leftarrow b(1)           \\
                             && c(2) \leftarrow b(2)                               
\end{array}
$$
Program $P_1$ is a guess\&check program and so $\Pi$ it can be rewritten as $\forall P_1^{\prime} \exists P_2^{\prime} : C^{\prime}$:
$$
\begin{array}{lllll}
     P_1^{\prime}            && P_2^{\prime}                   && C^{\prime} = \emptyset \\
     \{a(1);a(2)\}\leftarrow && b(1) \leftarrow unsat          &&                        \\
                             && b(2) \leftarrow unsat          &&                        \\
                             && c(1) \leftarrow b(1), \ unsat  &&                        \\
                             && c(2) \leftarrow b(2), \ unsat  &&                        \\
                             && unsat\leftarrow a(1), \ a(2)   &&
\end{array}
$$
The resulting program contains only one universal level, that is a trivial program and so it can be directly encoded in a QBF formula in CNF. However, well-founded optimization can be further applied but this is a corner case in which the combination of well-founded and guess check optimization results in larger programs.
Once we compute the well-founded of P2' (with the interface from previous level) we are unable to derive new knowledge, and all the rules of P2' are kept.
On the other hand, if we only apply the the well-founded simplification to P2 (with the interface from previous level) we derive b(1),b(2),c(1),c(2), and we are able to simplify all the rules.

\section{Proofs}\label{app:proofs}
\setcounterref{thm}{thm:phi:wf}
\addtocounter{thm}{-1}
\begin{thm}
\thmtextphiwf{}
\end{thm}

\begin{proof}
By means of $\CH^{\prime}$ possible models of $P_i$ are reduced to those that are coherent with models from previous levels by fixing the truth value of literals that have been determined by the well-founded operator and so $\as(G^{WF}_i)) \subseteq \as(G_i)$. 
If there exists $M\in \as(G_i)$ such that $M \notin \as(G^{WF}_i)$ then there exists some literal $l \in M$ such that $\sim l$ belongs to the well-founded model of $P_{i-1}$ and so $M$ is not coherent with models of previous levels. 
So, the program $P_i^{\prime}=P_i \cup \CH^{\prime}(G^{WF}_{i-1},P_i)$ preserves the coherence of $\Pi$.
From Proposition~\ref{prop:eq:residual} we know that $\as(P_i^{\prime}) = \as(R(P_i^{\prime}))$ and so from Theorem~\ref{th:original-translation} we can conclude that $\Phi^{\mathcal{WF}}(\Pi)$ is true iff $\Pi$ is coherent.
\end{proof}

\setcounterref{thm}{thm:omitted:forall}
\addtocounter{thm}{-1}
\begin{thm}
\thmtextomittedforall{}
\end{thm}

\begin{proof}
From Theorem \ref{th:original-translation} we know that $\Pi$ is coherent iff $\Phi(\Pi)$ is satisfiable. Now we observe that, by hypothesis, for all $k \in K$ it holds that $P_k$ is trivial, thus $\as(P_k)\mid_{Ext_k} = 2^{Ext_k}$. This implies that for all $k \in K$, $CNF(G_k)$ is satisfiable, and thus  $\phi_k \leftrightarrow \top$ holds in $\Phi(\Pi)$. Since $\phi_k \leftrightarrow  CNF(G_k) \leftrightarrow \top$ holds, then the k-th conjunct is satisfiable and can be omitted obtaining an equivalent formula:
$$\Phi(\Pi) = \boxplus_1 \cdots \boxplus_{n} \left(\bigwedge_{\substack{ i=1\\i\notin K}}^{n+1} (\phi_i \leftrightarrow \CNF(G_i))\right) \wedge \phi_c,$$  
$\boxplus_{i} = \exists x_i$ if $\Box_i=\exists^{st}$, and $\boxplus_{i} = \forall x_i$ otherwise, $x_i = var(\phi_i \leftrightarrow \CNF(G_i))$ if $i\notin K$, otherwise $x_i = Ext_i$. 
Moreover, we observe that for each $k \in K$, $\phi_k \wedge (\phi_{k+1} \odot_{k+1} \cdots)$, can be replaced by $\top \wedge (\phi_{k+1} \odot_{k+1} \cdots)$ which is equivalent to $\phi_{k+1} \odot_{k+1} (\cdots)$ and $\neg \phi_k \vee (\phi_{k+1} \odot_{k+1} \cdots)$, can be replaced by $\bot \vee (\phi_{k+1} \odot_{k+1} \cdots)$ which is equivalent to $(\phi_{k+1} \odot_{k+1} \cdots)$.
So, $\phi_c$ can be simplified by removing $\phi_k' \odot_k$ for each $k \in K$, obtaining an equivalent formula that is $\phi_c^K$.
Thus, by construction, $\Phi(\Pi)$ is equivalent to $\Phi^{K}(\Pi)$ and so $\Phi^K(\Pi)$ is satisfiable iff $\Pi$ is coherent.
\end{proof}

\setcounterref{thm}{thm:moved:equiv}
\addtocounter{thm}{-1}
\begin{thm}
\thmtextmovedequiv{}
\end{thm}

\begin{proof}
(I) First assume that $\Pi$ is incoherent. Obviously, if $\Pi$ is incoherent due to a $P_j$ with $j<i$, then $\Pi^{GC_i}$ is incoherent for the same reason. 

If $\Pi$ is incoherent due to $P_i$, in the following we will consider $P_i$ only (as it is the case when $i=1$), rather than $P_i\cup fix_{P_{i-1}}(N)$ for some answer set $N$ of the previous level (for $i>1$). All arguments transfer directly to the latter case.

For any $M\in \as(P_i)$ we know from Proposition~\ref{prop:stable:guess:check} that $M=M_G \cup M_C$ where $M_G \in \as(G_{P_i})$ and $M_C \in \as(C_{P_i} \cup fix_{G_{P_i}}(M_G))$. If $M$ is the reason for incoherence in $\Pi$, we will show that then $M_G$ is a reason for incoherence in $\Pi^{GC_i}$. We distinguish three cases.

(1) If $i=n$, there is $M\in \as(P_i)$ with $C\cup fix_{P_i}(M)$ incoherent. Here, $M_C$ is also the unique answer set of $\tau(u,P_i) \cup fix_{G_{P_i}}(M_G)$, but since $C\cup fix_{P_i}(M)$ is incoherent, $M_C$ does not satisfy $\rho(u,C)$ either (as $u$ is false in it). It follows that there is $M_G\in \as(G_{P_i})$ such that $\sigma(u,P_i,C) \cup fix_{G_{P_i}}(M_G)$ is incoherent, and hence $\Pi^{GC_i}$ is incoherent. 

(2) If $i=n-1$, there is $M\in \as(P_i)$ such that there is no $M' \in \as(P_n\cup fix_{P_i}(M))$ such that $C \cup fix_{P_n}(M')$ is incoherent.
 Also here, $M_C$ is also the unique answer set of $\tau(u,P_i) \cup fix_{G_{P_i}}(M_G)$, and each $M' \in \as(P_n\cup fix_{P_i}(M))$ satisfies $\rho(u,P_n)$ as well, because $u$ is false in $M'$, which means that $\as(P_n\cup fix_{P_i}(M)) = \as(\sigma(u,P_i,P_n)\cup fix_{G_{P_i}}(M_G))$. Finally, since $u$ is false in each $M'$, from $C \cup fix_{P_n}(M')$ being incoherent we also get that $\rho(u,C )\cup fix_{P_n}(M')$ is incoherent. Then we have $M_G\in \as(G_{P_i})$ such that there is no $M' \in \as(\sigma(u,P_i,P_n)\cup fix_{G_{P_i}}(M_G))$ such that $\rho(u,C) \cup fix_{P_n}(M')$ is incoherent, hence $\Pi^{GC_i}$ is incoherent. 

(3) If $i<n-1$, there is $M\in \as(P_i)$ such that there is no $M' \in \as(P_{i+1} \cup fix_{P_i}(M))$ such that $\Pi'_{P_{i+1},M'}$ is incoherent, where $\Pi'$ is the suffix of $\Pi$ starting at $P_{i+2}$. Also here, $M_C$ is also the unique answer set of $\tau(u,P_i) \cup fix_{G_{P_i}}(M_G)$, and each $M' \in \as(P_{i+1}\cup fix_{P_i}(M))$ satisfies $\rho(u,P_{i+1})$ as well, because $u$ is false in $M'$, which means that $\as(P_{i+1}\cup fix_{P_i}(M)) = \as(\sigma(u,P_i,P_{i+1})\cup fix_{G_{P_i}}(M_G))$. Also, observe that since $u$ is false in each of these $M'$, $\as(P_{i+2} \cup fix_{P_{i+1}}(M')) = \as(P_{i+2} \cup \{\leftarrow u\} \cup fix_{P_{i+1}}(M'))$. So there is $M_G\in \as(G_{P_i})$ such that there is no $M' \in \as(\sigma(u,P_i,P_{i+1})\cup fix_{G_{P_i}}(M_G))$ such that $\Pi''_{P_{i+1},M'}$ (the suffix of $\Pi^{GC_i}$ starting at $P_{i+2} \cup \{\leftarrow u\}$) is incoherent, so $\Pi^{GC_i}$ is incoherent. 

(II) Now assume that $\Pi$ is coherent. As above, in the following we will consider $P_i$ only (as it is in the case when $i=1$), rather than $P_i\cup fix_{P_{i-1}}(N)$ for some answer set $N$ of the previous level (for $i>1$).

Here we have to show that from coherence for each $M\in \as(P_i)$ in $\Pi$ coherence for each $M_G \in \as(G_{P_i})$ follows. If (case a) there is an $M_C \in \as(C_{P_i} \cup fix_{G_{P_i}}(M_G))$ (the unique answer set), this follows quite easily because $M=M_G\cup M_C$ due to Proposition~\ref{prop:stable:guess:check} and $M_C$ is the unique answer set of $\tau(u,P_i)\cup fix_{G_{P_i}}(M_G)$, in which $u$ is false. If (case b) $\as(C_{P_i} \cup fix_{G_{P_i}}(M_G))$ is incoherent, then there is a single answer set $M_u$ of $\tau(u,P_i)\cup fix_{G_{P_i}}(M_G)$, in which $u$ is true.
We distinguish three cases.

(1) If $i=n$, for any $M \in \as(P_i)$ the program $C\cup fix_{P_i}(M)$ is coherent; let $X$ be one of its answer sets. In case a, $M_C \cup X$ is an answer set of $\sigma(u,P_i,C) \cup fix_{G_{P_i}}(M_G)$. In case b, $M_u$ is an answer set of $\sigma(u,P_i,C)\cup fix_{G_{P_i}}(M_G)$ (as all rules in $\rho(u,C)$ are satisfied by $M_u$ due to $u$ being true). In both cases, $\Pi^{GC_i}$ is coherent.

(2) If $i=n-1$, for any $M\in \as(P_i)$ there is an $M' \in \as(P_n\cup fix_{P_i}(M))$ such that $C \cup fix_{P_n}(M')$ is coherent with an answer set $X$. In case a, $M_C \cup M'$ is an answer set of $\sigma(u,P_i,P_n) \cup fix_{G_{P_i}}(M_G)$, and $C \cup fix_{P_n}(M_C\cup M')$ is coherent with the answer set $X$. In case b, $M_u$ is an answer set of $\sigma(u,P_i,P_n) \cup fix_{G_{P_i}}(M_G)$ (as all rules in $\rho(u,C)$ are satisfied by $M_u$ due to $u$ being true). But then $M_u$ is also an answer set of $C \cup fix_{P_n}(M_u)$, which is therefore coherent. In both cases, $\Pi^{GC_i}$ is coherent.

(3) If $i<n-1$, for any $M\in \as(P_i)$ there is an $M' \in \as(P_{i+1} \cup fix_{P_i}(M))$ such that $\Pi'_{P_{i+1},M'}$ is coherent, where $\Pi'$ is the suffix of $\Pi$ starting at $P_{i+2}$. 
In case a, $M_C \cup M'$ is an answer set of $\sigma(u,P_i,P_{i+1}) \cup fix_{G_{P_i}}(M_G)$, and $\Pi''_{P_{i+1},M_C\cup M'}$ (the suffix of $\Pi^{GC_i}$ starting at $P_{i+2} \cup \{\leftarrow u\}$) is coherent since $M_C\cup M' = M'$ and $u$ is false in $M'$. 
In case b, $M_u$ is an answer set of $\sigma(u,P_i,P_{i+1}) \cup fix_{G_{P_i}}(M_G)$ (as all rules in $\rho(u,P_{i+1})$ are satisfied by $M_u$ due to $u$ being true). 
But then $P_{i+2} \cup \{\leftarrow u\} \cup fix_{P_{i+1}}(M_u)$ has no answer sets (because of $u$), so $\Pi''_{P_{i+1},M_C\cup M'}$ is trivially coherent. In both cases, $\Pi^{GC_i}$ is coherent.
\end{proof}

\setcounterref{thm}{thm:allGC}
\addtocounter{thm}{-1}
\begin{thm}
\thmtextallGC{}
\end{thm}

\begin{proof}
Observe that, by definition, $\Pi_n$ is such that all of its universally quantified subprograms are trivial (contain only choice rules). Let $K=\{k | k\in [1,\dots, n] \wedge \Box_i=\forall^{st} \text{ in } \Pi_n\}$,  from Theorem~\ref{thm:omitted:forall} it follows that $\Pi_n$ is coherent iff $\Phi^K(\Pi_n)$ is satisfiable. Moreover, from Proposition~\ref{prop:trivial:forall}, we have that $\Phi^K(\Pi_n)$ is equivalent to $\Phi^K_{CNF}(\Pi_n)$. 
From Theorem \ref{thm:moved:equiv} we have that, $\Pi$ is coherent iff $\Pi_n$ is coherent; since $\Pi_n$ is coherent iff $\Phi^K(\Pi_n)$ is satisfiable, and $\Phi^K(\Pi_n)$ is equivalent to $\Phi^K_{CNF}(\Pi_n)$, the thesis follows.
\end{proof}

\end{document}